\documentclass[twoside]{article}

\usepackage[accepted]{aistats2022}

\usepackage{tikz}
\newcommand*\circled[1]{\tikz[baseline=(char.base)]{
            \node[scale=0.8, shape=circle,draw,inner sep=2pt] (char) {#1};}}

\usepackage{algorithm}
\usepackage{algorithmic}

\usepackage[symbol]{footmisc}

\usepackage[utf8]{inputenc} %
\usepackage[T1]{fontenc}    %
\usepackage{url}            %
\usepackage{booktabs}       %
\usepackage{amsfonts}       %
\usepackage{nicefrac}       %
\usepackage{microtype}      %
\usepackage{lipsum}
\usepackage{amsmath}
\usepackage{amsthm}
\usepackage{mathtools}
\usepackage{doi}

\usepackage{natbib}
\renewcommand{\refname}{References}
\makeatletter
\renewcommand{\bibsection}{%
   \subsubsection*{\refname%
            \@mkboth{\MakeUppercase{\refname}}{\MakeUppercase{\refname}}%
   }
}
\makeatother

\usepackage{color, colortbl}
\usepackage{pifont}
\usepackage{thmtools}

\usepackage{dsfont}  %

\usepackage{thm-restate}
\usepackage[inline,shortlabels]{enumitem}

\definecolor{mydarkgreen}{RGB}{39,130,67}
\newcommand{\green}{\color{mydarkgreen}}
\definecolor{mydarkred}{RGB}{192,47,25}
\newcommand{\red}{\color{mydarkred}}
\definecolor{darkblue}{rgb}{0.0,0.0,0.65}
\definecolor{darkred}{rgb}{0.68,0.05,0.0}
\definecolor{darkgreen}{rgb}{0.0,0.29,0.29}
\definecolor{darkpurple}{rgb}{0.47,0.09,0.29}

\newcommand{\partB}[1]{{\color{darkblue}#1}}

\newcommand{\greencheckmark}{\green\checkmark}

\definecolor{mydarkblue}{rgb}{0,0.08,0.45}

\usepackage[nameinlink]{cleveref}
\crefname{equation}{Eq.}{equations}
\Crefname{equation}{Equation}{equations}
\crefname{table}{Tab.}{tables}
\Crefname{table}{Table}{tables}
\crefname{figure}{Fig.}{figures}
\Crefname{figure}{Figure}{figures}
\crefname{algorithm}{Alg.}{algorithms}
\Crefname{algorithm}{Algorithm}{algorithms}
\crefname{section}{Sec.}{sections}
\Crefname{section}{Section}{sections}

\usepackage{hyperref}
\hypersetup{
   colorlinks = true,
   citecolor  = darkblue,
   linkcolor  = darkred,
   filecolor  = darkblue,
   urlcolor   = darkblue,
 }

\setcitestyle{authoryear,round,citesep={;},aysep={,},yysep={;}}

\newcommand{\grad}{\nabla}
\newcommand{\dual}[1]{#1^\ast}
\newcommand{\innerp}[1]{\langle#1\rangle}
\newcommand{\domain}{\mathcal W}
\newcommand{\constset}{\mathcal K}
\DeclareMathOperator{\lmo}{\textsc{lmo}}
\DeclareMathOperator*{\argmin}{arg\,min}
\DeclareMathOperator{\prox}{\mathrm{prox}}

\DeclareMathOperator{\dist}{dist}
\newcommand{\Expect}{\mathbb{E}}
\newcommand{\N}{\mathbb{N}}
\newcommand{\R}{\mathbb{R}}
\newcommand{\betak}{{\beta_{k}}}
\newcommand{\betakminusone}{{\beta_{k-1}}}

\DeclareMathOperator{\Tr}{Tr}

\usepackage{amssymb}
\newcommand{\defeq}{\triangleq}

\newcommand{\ellone}[1]{\|#1\|_1}
\newcommand{\elltwo}[1]{\|#1\|_2}
\newcommand{\norm}[1]{\|#1\|}
\DeclarePairedDelimiter{\floor}{\lfloor}{\rfloor}

\newcommand{\optimal}[1]{#1^{\star}}
\newcommand{\bigO}{\mathcal O}
\newcommand{\proj}{\mathrm{proj}}
\newcommand{\ip}[2]{\langle{#1},{#2}\rangle}

\newcommand{\bb}{b}
\newcommand{\x}{x}
\newcommand{\y}{y}
\newcommand{\z}{z}
\newcommand{\w}{w}
\newcommand{\uu}{u}
\newcommand{\vv}{v}
\newcommand{\W}{W}
\newcommand{\C}{C}
\newcommand{\X}{X}
\newcommand{\A}{A}
\newcommand{\va}{a}
\newcommand{\vx}{x}

\newcommand{\gideon}[1]{{{\color{red}[\textit{gideon}: #1]}}}

\declaretheorem[name=Theorem,numberwithin=section]{thm}
\declaretheorem[name=Corollary,numberwithin=section]{corollary}

\declaretheorem[name=Lemma,numberwithin=section]{lemma}

\newcommand{\xmark}{\ding{55}}

\newcommand{\redx}{\red\large\xmark}

\usepackage{xspace}
\newcommand{\Prob}{\hyperref[main:template_def]{Problem~(1)}\xspace}

\begin{document}

\runningtitle{Faster One-Sample SCGM for Composite Convex Minimization}

\runningauthor{Dresdner, Vladarean, R\"atsch, Locatello, Cevher, Yurtsever}

\twocolumn[
\aistatstitle{Faster One-Sample Stochastic Conditional Gradient Method for Composite Convex Minimization}
\aistatsauthor{Gideon Dresdner \And Maria-Luiza Vladarean \And Gunnar R\"atsch }%
\aistatsaddress{ETH Z\"urich, Switzerland \And EPFL, Switzerland  \And ETH Z\"urich, Switzerland}

\aistatsauthor{Francesco Locatello \And Volkan Cevher \And Alp Yurtsever}
\aistatsaddress{Amazon Web Services \And EPFL, Switzerland \And Ume\aa\ University, Sweden}
]

\begin{abstract}

We propose a stochastic conditional gradient method (CGM) for minimizing convex finite-sum objectives formed as a sum of smooth and non-smooth terms. Existing CGM variants for this template either suffer from slow convergence rates, or require carefully increasing the batch size over the course of the algorithm's execution, which leads to computing full gradients. In contrast, the proposed method, equipped with a stochastic average gradient (SAG) estimator, requires only one sample per iteration. Nevertheless, it guarantees fast convergence rates on par with more sophisticated variance reduction techniques. In applications we put special emphasis on problems with a large number of separable constraints. Such problems are prevalent among semidefinite programming (SDP) formulations arising in machine learning and theoretical computer science. We provide numerical experiments on matrix completion, unsupervised clustering, and sparsest-cut SDPs.
 \end{abstract}

\section{INTRODUCTION}

Consider the following composite finite-sum template:
\begin{equation}\label{main:template_def}
    \min_{\w\in\domain}
    \left\{F(\w)\defeq \frac{1}{n}\sum_{i=1}^n f_i(\x_i^T \w) + g(\A\w)\right\}.
\end{equation}
$\domain \subset \mathbb R^d$ is a compact and convex set, each $f_i:\mathbb R\to\mathbb R$ is convex and $L_f$-smooth (i.e., its derivative is Lipschitz continuous with constant $L_f$), $A$ is an~$m \times d$ matrix, and $g: \mathbb R^m \to\mathbb R\cup\{+\infty\}$ is convex but possibly non-smooth. The function $g(\A\w)$ can capture constraints of the form $\A\w = b$ (or $\A\w \in \mathcal{K}$, for closed, convex sets $\mathcal{K} \subseteq \mathbb{R}^m$) via indicator functions $\delta_{\{b\}}$ (resp.,~$\delta_{\mathcal{K}}$ which takes $0$ for all points in $\mathcal{K}$ and $+\infty$ everywhere else). Throughout, we assume that $g$ is either Lipschitz continuous or an indicator function.

We study conditional gradient methods (CGM, also known as the Frank-Wolfe Algorithm) tailored for \Prob. 
For computational efficiency, we suppose linear minimization over $\domain$ is easy. We separately focus on two specific settings of $g$: 
\begin{enumerate}[leftmargin=2.25em,topsep=0em,itemsep=0em,partopsep=1ex,parsep=1ex]
    \item[(\texttt{S1})]  \hypertarget{setting:v1} $g$ admits an efficient prox-operator,
    \item[(\texttt{S2})] \hypertarget{setting:v2} $g$ is a finite-sum of the form $g \defeq \frac{1}{m} \sum_{i = 1}^m g_i(a_i^T \w)$, where each $g_i : \mathbb{R} \to \mathbb{R}\cup \{+\infty\}$ is convex and $a_i^T$ is the $i$-th row of $A$. 
    This \emph{separable} finite-sum structure allows us to tackle $g$ stochastically and therefore more efficiently when $m$ is large. %
\end{enumerate}

Our problem template covers a variety of applications in machine learning, statistics and signal processing, including the finite-sum formulations that arise in M-estimation and empirical risk minimization problems.

\subsection*{Application Focus:~Strongly Constrained SDPs}

A particular example of our model problem is the standard semidefinite programming (SDP) template:
\begin{equation}\label{main:basic_sdp_def} %
\begin{aligned}
& \min_{\W \in \mathbb S^{d \times d}_+} & & \ip{\X}{\W}  \\
& \mathrm{subj.\ to}  & & \ip{\A_i}{\W} \triangleleft \bb_i, ~~ i=1,\ldots, m,
\end{aligned}
\end{equation}
where $\mathbb S^{d \times d}_+$ denotes the set of symmetric positive semidefinite matrices, $\X \in \mathbb S^{d \times d}$ is the symmetric cost matrix, $(\A_i, b_i) \in \mathbb{S}^{d \times d} \times \mathbb{R}$ characterize the constraints, and `$\triangleleft$' represents either equality `$=$' or inequality~`$\le$' operations. 

SDPs are ubiquitous in theoretical computer science. Examples include relaxations of combinatorial optimization problems such as maximum cut \citep{goemans1995improved}, quadratic assignment \citep{zhao1998semidefinite}, and sparsest cut \citep{arora2009expander}. SDPs are also found in machine learning problems such as matrix completion \citep{alfakih1999solving}, unsupervised clustering \citep{kulis2007fast}, certifying robustness of neural networks \citep{Raghunathan18} and estimating their Lipschitz constants \citep{latorre2020lipschitz}.

The remarkable flexibility of SDPs  comes at the cost of severe 
computational challenges. %
The cone constraint itself poses a major challenge for a majority of the first-order methods because projection onto positive semidefinite cone requires expensive eigen-decompositions. 
CGM is popular in this setting (see \citet{hazan2008sparse, jaggi2010simple, garber2016faster, yurtsever2018conditional}) since it avoids projection by leveraging the so-called linear minimization oracle (lmo) which computes only the top eigenvectors rather than the full spectrum. 
Additionally, CGM is also used to reduce storage cost \citep{yurtsever15b,freund2017extended,yurtsever2021scalable}, which is often a critical bottleneck for solving SDPs in large scale.  

However, scalable approaches to solving SDPs with a large number of constraints, which we term as strongly-constrained SDPs, remain largely unexplored. This gap can be bridged by developing CGM variants which handle linear constraints in a randomized fashion.

\paragraph{Contributions.} We propose a new CGM variant for convex finite-sum problems. 
The proposed method extends the recent work on stochastic Frank-Wolfe \citep{negiar2020stochastic} to the composite template in \Prob. %
In particular:
\begin{itemize}[$\triangleright$,topsep=-1pt,itemsep=0ex,partopsep=1ex,parsep=1ex]
    \item In (\hyperlink{setting:v1}{$\texttt{S1}$}), our algorithm finds an $\varepsilon$-suboptimal solution after $\mathcal O(\varepsilon^{-2})$ iterations (see Optimality Conditions, \cref{sec:preliminaries} for the definition of $\varepsilon$-suboptimal). %
    \item In (\hyperlink{setting:v2}{$\texttt{S2}$}), our algorithm finds an $\varepsilon$-suboptimal solution after $\mathcal O(\varepsilon^{-2})$ iterations, matching the iteration complexity in \citet{vladarean2020conditional}. However, we achieve this rate without using an increasing batch-size strategy. Thus, our algorithm enjoys a total cost of $\mathcal O(\varepsilon^{-2}d)$ which is independent of $m$. In contrast, the cost in \citet{vladarean2020conditional} is~$\mathcal O(\varepsilon^{-2}dm)$.
\end{itemize}

Finally, we present numerical experiments on matrix completion, $k$-means clustering, and sparsest cut problems. 
In these experiments, the proposed algorithm performs on par with the state-of-the-art variance reduced CGM variants. Importantly, however, our algorithm does not require computing full gradients or increasing the batch size. %

\begin{table*}
\begin{small}
\begin{center}
\begin{tabular}{r|c|c|c|c}
Algorithm       
& Reference                          
& Iteration Complexity        
& Total Cost
& Fixed Batch Size
\\ 
 \toprule
HCGM, CGAL
& \citet{yurtsever2018conditional, pmlr-v97-yurtsever19a}
&  $\mathcal{O}(\varepsilon^{-2})$ 
& $\mathcal O(\varepsilon^{-2}d \, \max\{n,m\})$
& N/A
\\ \midrule
SHCGM
& \citet{locatello_stochastic_2019}
& $\mathcal{O}(\varepsilon^{-3})$ 
& $\mathcal O(\varepsilon^{-3}d\,m)$
& N/A
\\
MOST-FW
& \citet{akhtar2021zeroth}
& $\mathcal{O}(\varepsilon^{-2})$ 
& $\mathcal O(\varepsilon^{-2}d\,m)$
& N/A
\\
{H-SAG-CGM v1}
& \emph{This Paper}
& ${\mathcal{O}(\varepsilon^{-2})}$ 
& ${\mathcal O(\varepsilon^{-2}d\,m )}$
& N/A
\\ \midrule
H-1SFW 
& \citet{vladarean2020conditional}
&  $\mathcal{O}(\varepsilon^{-6})$ 
&  $\mathcal O(\varepsilon^{-6}d)$
& \greencheckmark
\\
H-SPIDER-FW
& \citet{vladarean2020conditional}
&  $\mathcal{O}(\varepsilon^{-2})$ 
&  $\mathcal O(\varepsilon^{-2}d \, m)$
& \redx
\\
MOST-FW\textsuperscript{+}
& \citet{akhtar2021zeroth}
&  $\mathcal{O}(\varepsilon^{-4})$ 
&  $\mathcal O(\varepsilon^{-4}d)$
& \greencheckmark
\\
{H-SAG-CGM v2}
& \emph{This Paper}
& ${\mathcal{O}(\varepsilon^{-2})}$ 
&  ${\mathcal O(\varepsilon^{-2}d)}$
& \greencheckmark
\\
\end{tabular}
\end{center}
\end{small}
\caption{
This table presents asymptotic costs of finding an $\varepsilon$-suboptimal solution to a given problem, \textit{i.e.}, we treat problem parameters $d$, $n$ and $m$ as constants and characterize the behavior as $\varepsilon \to 0$. $\mathcal{O}$ notation hides the parameters $L_f$, $\|A\|$, $D_\domain$, and the absolute constants. We tailor the cost of existing methods for \Prob, their cost for other problems can be different. The last column indicates whether the algorithm has increasing or fixed batch size.}
\label{tab:rates}
\end{table*}

%

\section{RELATED WORK}

\paragraph{CGM for Smooth Objectives.}
CGM is introduced by \citet{frank1956algorithm} for minimizing a convex quadratic function over a polytope. 
Later, the analysis is extended to general convex smooth functions and arbitrary convex and compact sets by \citet{levitin1966}. 
\citet{clarke1990optimization} and \citet{hazan2008sparse} propose CGM as an effective method to tackle simplex and spectrahedron constraints respectively. 
We refer to \citet{jaggi_revisiting_2013} for an excellent survey on the efficiency of CGM for machine learning applications. 

The last decade has witnessed a surge of interest in the CGM framework for machine learning applications which has prompted researchers to study stochastic extensions of CGM. 
Unlike gradient descent, CGM does not immediately work when the gradient in the algorithm is replaced with an unbiased stochastic gradient estimator with bounded variance. 
To address this problem, several stochastic CGM variants have been proposed by combining CGM with existing variance reduction techniques \citep{Reddi2016,hazan_variance-reduced_2017,mokhtari_stochastic_2018,pmlr-v97-yurtsever19b,Shen2019,zhang2020one} and more recently in \citet{negiar2020stochastic}. 

In general, the convergence rate of an algorithm is determined by the stochastic gradient estimator. \citet{hazan_variance-reduced_2017} develop an  estimator with small variance, resulting in a fast $\mathcal O(\varepsilon^{-3/2})$ iteration complexity but at the cost of exponentially increasing batch sizes. \citet{mokhtari_stochastic_2018} and \citet{zhang2020one} maintain a constant batch size but have slower convergence rates of $\mathcal{O}(\epsilon^{-3})$ and $\mathcal{O}(\epsilon^{-2})$, respectively. We refer to \citet{pmlr-v97-yurtsever19b} for a detailed comparison of the existing stochastic CGM variants. 

Our work draws from \citep{negiar2020stochastic} where the authors propose a stochastic CGM with an iteration complexity of~$\bigO(\varepsilon^{-1})$ which is on par with deterministic CGM. This is achieved by assuming a separable finite-sum model and using the Stochastic Average Gradient (SAG) estimation technique \citep{schmidt2017minimizing}.

\paragraph{CGM for Composite Objectives.}
CGM is not directly applicable to problems with a non-smooth objective (see Section~2 in \citep{Nesterov:2018aa} for a counter-example). 
\citet{lan2013complexity} tackle this problem in the case of Lipschitz continuous non-smooth functions by combining CGM with Nesterov smoothing \citep{nesterov05}. \citet{yurtsever2018conditional} further extend it for indicator functions through a quadratic penalty technique, which they call Homotopy CGM.

\citet{locatello_stochastic_2019} extend Homotopy CGM to stochastic objectives but only for the case in which the non-smooth part $g$ is deterministic. More recently, \citet{vladarean2020conditional} proposed new variants that can handle stochastic constraints. They provide algorithms for an arbitrary number of constraints under minimal assumptions. However, for the common practical setting of a finite number of constraints, their algorithm requires full passes over the constraints.

This paper works in the same vein by proposing a randomized algorithm for the finite-sum template in \Prob. Our algorithm for deterministic $g$ in (\hyperlink{setting:v1}{$\texttt{S1}$}) outperforms the method of \citet{locatello_stochastic_2019} both in theory and in practice. Our algorithm for separable $g$ in (\hyperlink{setting:v2}{$\texttt{S2}$}) performs on par with the methods described in \citep{vladarean2020conditional}. However, in contrast to the previous work, it maintains a constant batch size.

After submitting this paper, we became aware of the recent work of \citet{akhtar2021zeroth}. They study a similar problem and also propose an algorithm with two variants to address the cases of deterministic and stochastic $g$. In the case of deterministic $g$, the cost of their algorithm is $\mathcal{O}(\varepsilon^{-2}dm)$; the same as our method. However, in the case of stochastic $g$, their method's cost is $\mathcal O(\varepsilon^{-4}d)$. In contrast, our algorithm achieves $\mathcal O(\varepsilon^{-2}d)$ by taking advantage of the separable finite-sum structure.

\textbf{Proximal Methods.}
A growing body of work aims to address strongly constrained problems through proximal methods in various settings \citep{patrascu2017nonasymptotic,fercoq2019almost,mishchenko2019stochastic,xu2020primal}. These algorithms process a random subset of constraints at each iteration and converge to a feasible point asymptotically, similar to \citep{vladarean2020conditional} and the algorithm that we propose in this paper. However, when applied to SDPs, proximal methods require a costly eigenvalue decomposition at each iteration. Hence, these methods are not practical for solving SDPs in large scale.

\vspace{1em}
\textbf{Primal vs.\ Dual Problem.}
When there are many constraints, solving the dual problem can be more plausible from a computational perspective. However, converting a dual solution to a primal solution is a non-trivial problem itself, especially in large-scale setting where we are restricted from using projection or proximal operators. Moreover, since our problem is stochastic, we can expect finding only a rough estimate of the dual solution. In this work, we assume that we are interested in the primal variable and that it is large. To this end, we focus on solving the primal formulation.

\section{PRELIMINARIES}\label{sec:preliminaries}

\paragraph{Notation.}
The operator norm of a matrix~$\A$ is written~$\|A\|$ and the Euclidean inner-product is denoted $\innerp{\cdot,\cdot}$. 
We define the diameter of~$\domain$ as
\begin{align}
D_\domain = \max_{\x,\y\in\domain}\elltwo{\x - \y}
\end{align}
and the $\ell_1$ and $\ell_\infty$ diameters with respect to the column space of a matrix $M$ as 
\begin{align}
D_1(M) & \defeq \max_{u,v\in\domain} \|M(u-v)\|_1 \\
D_\infty(M) & \defeq \max_{u,v\in\domain} \|M(u-v)\|_\infty.
\end{align}
The linear minimization oracle of set $\domain$ is given by 
\begin{align}
\lmo_\domain(\vv) \defeq \argmin_{\uu\in\domain}\innerp{\uu,\vv}.
\end{align} 
The proximal operator of $g:\mathbb{R}^m \to \mathbb{R} \cup \{+\infty\}$ is 
\begin{align}
    \prox_g(\z) \defeq \argmin_{\y\in \mathbb{R}^m} ~ g(\y) + \frac{1}{2}\elltwo{\y-\z}^2.
\end{align} 
When $g$ is the indicator function of a convex set $\constset$, its proximal operator is equal to the Euclidean projection, $\prox_{\delta_{\constset}}(\z) = \proj_{\constset}(\z)$. 

\paragraph{Assumption.} When $g$ is an indicator function we assume that strong duality holds. Slater’s condition is a well-known sufficient condition for strong duality.

\paragraph{Optimality Conditions.}
We denote a solution of \Prob by $\optimal \w$:
\begin{align}
\optimal F \defeq F(\optimal \w) \leq F(\w), \quad \forall \w \in \domain.
\end{align}
If $g$ is continuous valued on $\domain$, we say that $\w_k \in \domain$ is an $\varepsilon$-suboptimal solution when it satisfies 
\begin{align}
\Expect F(\w_k) - \optimal F \leq \varepsilon. 
\end{align}
If $g = \delta_{\constset}$ is an indicator function, the $F(\w_k) - \optimal F$ can be $+\infty$ even when $\w_k$ is arbitrarily close to a solution. To this end, we relax the definition of an $\varepsilon$-suboptimal solution in this case and say that $\w_k \in \domain$ is an $\varepsilon$-suboptimal solution of \Prob if it satisfies 
\begin{align}
    |\Expect f(\w_k) - \optimal F| \leq \varepsilon \quad \text{and} \quad \Expect [\dist(\A\w_k;\constset)] \leq \varepsilon.
\end{align}
Our algorithm guarantees at every iteration that $\w_k$ is in $\domain$ and asymptotically that $A \w_k \in \constset$.

\subsection{Smoothing}

\label{subsec:hom-smth}
Building on the existing Homotopy CGM framework \citep{yurtsever2018conditional,locatello_stochastic_2019,vladarean2020conditional}, we use the smoothing technique of \citet{nesterov05} and its extension to indicator functions as studied in~\citet{trandinh2018smooth}. Specifically, given a convex (possibly non-smooth) function $g$, its approximation is defined as
\begin{align}
    g_\beta(z) \defeq \sup_y \innerp{y,z} - \dual{g}(y) - \frac{\beta}{2}\|y\|^2,
\end{align}
where $\dual{g}(y) \defeq \sup_x \innerp{x,y} - g(x)$ is the Fenchel conjugate of $g$. Importantly, $g_\beta$ is $\frac{1}{\beta}$-smooth \citep{nesterov05}. When $g=\delta_{\constset}$ for some closed and convex set~$\constset$, its approximation becomes $g_{\beta}(z) = \frac{1}{2\beta}\mathrm{dist}(z, \constset)^2$.
If $g$ allows for an efficient $\mathrm{prox}$ operator, we can compute the gradient of $g_\beta$ as
\begin{align}
\nabla g_{\beta}(\A\w) =\beta^{-1} \left(\A\w - \prox_{\beta g}(\A\w)\right).
\end{align}
\begin{algorithm}[!t]
\caption{H-SAG-CGM}
\label{alg:main}
\begin{algorithmic}[1]
\STATE {\bfseries Input:} $\beta_0>0, ~\w_0\in\domain,~ \alpha_0\in\mathbb R^n,~\gamma_0\in\mathbb R^m,$\\
\quad$v_0^f\in\R^d,~v_0^g\in\mathbb R^d$
\FOR{$k=1,2,\ldots$}
\STATE $\eta_k = \frac{2}{k+1}$
\STATE $\beta_k = \beta_0/\sqrt{k+1}$
\STATE Sample $j\sim\operatorname{Uniform}[1,2,\ldots,n]$
\STATE $\alpha_{k,i} = \begin{cases}\frac{1}{n}f_j'(\vx_j^T \w_{k})& i = j \\ \alpha_{k-1,i} & i\ne j\end{cases}$
\STATE $v_k^f = v_{k-1}^f + (\alpha_{k,j}-\alpha_{k-1,j})\x_i$
\STATE $v_k^g \leftarrow \textbf{use \hyperref[alg:v1]{Variant 1} or \hyperref[alg:v2]{Variant 2}}$
\STATE $v_k = v_k^f + v_k^g$
\STATE $s_k = \lmo_\domain(v_k)$\\
\STATE $w_{k+1} = w_{k} + \eta_k (s_k - w_{k})$
\ENDFOR
\end{algorithmic}
\end{algorithm}

\setcounter{algorithm}{0}

\floatname{algorithm}{Variant}
\begin{algorithm}[!t]
\caption{Non-separable Constraints}
\label{alg:v1}
\begin{algorithmic}[1]
\RETURN {$\frac{1}{\betak}A^T(Aw_k - \mathrm{prox}_{\betak g}(Aw_k))$}\\
\end{algorithmic}
\end{algorithm}
\floatname{algorithm}{Algorithm}

\floatname{algorithm}{Variant}
\begin{algorithm}[!t]
\caption{Randomized Constraints}
\label{alg:v2}
\begin{algorithmic}[1]
\STATE Sample $l\sim\operatorname{Uniform}[1,2,\ldots,m]$
\STATE$\gamma_{k,q} = \begin{cases} \frac{1}{m} g_{\beta_k, l}'(\va_l^T\w_k) & q = l\\ \gamma_{k-1,q} & q\ne l\end{cases}$
\RETURN $v_{k-1}^g + (\gamma_{k,l}-\gamma_{k-1,l})\va_l$
\end{algorithmic}
\end{algorithm}

\floatname{algorithm}{Algorithm}
\setcounter{algorithm}{1} %

\section{ALGORITHM \& CONVERGENCE}
\label{sec:algs_and_conv}

\subsection{Stochastic Homotopy-Based CGM for Separable Problems}

First, we transform the objective in \Prob using the smoothing technique summarized in \Cref{subsec:hom-smth} to obtain the following smooth surrogate objective:
\begin{align}
\label{H-OPT}
    F_\beta(\w) \defeq \frac{1}{n} \sum_{i=1}^n f_i(\x_i^T\w) + g_\beta(\A\w).
\end{align}
In particular, if we consider (\hyperlink{setting:v2}{$\texttt{S2}$}) in which $g$ is separable, then the smooth approximation $g_\beta$ is also separable:
\begin{align}
{g_\beta (Aw) = \frac{1}{m}\sum_{j=1}^m g_{\beta, j}(a_j^Tw)}.
\end{align}
This will allow for a fully randomized algorithm (H-SAG-CGM/v2) which can tackle strongly constrained SDPs with a non-increasing batch size.

The fundamental mechanism of homotopy CGM is to enforce a theoretically-determined schedule for $\beta_k$ such that $F_\betak\to F$ asymptotically. Broadly speaking, stochastic homotopy CGMs perform these three steps at each iteration:
\begin{enumerate}[(1),topsep=-1pt,itemsep=0ex,partopsep=1ex,parsep=1ex]
    \item Compute a gradient estimator $v_k$ of the smooth surrogate function $F_{\beta_k}$ (lines~5-8 of \cref{alg:main}, implemented in \hyperref[alg:v1]{Variant 1} and \hyperref[alg:v2]{Variant 2}).
    \item Perform a conditional gradient update by solving $\lmo_\domain(v_k)$ (\cref{alg:main}, \mbox{line 10}) and moving the current estimate towards this solution~(\cref{alg:main}, \mbox{line 11}).
    \item Decrease $\beta_k$ to enforce feasibility (line 4) and go to Step~(1).
\end{enumerate} 
\vspace{0.5em}

The main contribution of our algorithm is Step~(1) where we use a SAG estimator for $f$ and either the full gradient of $g_{\beta_k}$ (\hyperref[alg:v1]{Variant 1}) or another SAG estimator for $g_{\beta_k}$ (\hyperref[alg:v2]{Variant 2}). This key innovation over previous work in \citet{vladarean2020conditional} yields comparable, state-of-the-art complexity bounds without requiring full passes over the set of constraints. Then, Step~(2) comes from the classical CGM and Step~(3) is the homotopy smoothing step from \citet{yurtsever2018conditional}.

In the following section, we give an overview of the theoretical analysis.

\subsection{Analysis of Stochastic homotopy CGMs}

The analysis is composed of two main parts. 
First, %
we establish the convergence rate for the smoothed-gap%
\begin{align}
    S_{\beta_k}(w_{k+1}) \defeq \Expect [F_{\beta_k}(\w_{k+1}) - \optimal F].
\end{align}

Then, in the second part that we present in \Cref{sec:convergence-rates}, we translate convergence of the smoothed-gap $S_{\beta_k}$ into guarantees for the original problem based on the techniques described in \citep{trandinh2018smooth}.

For the first part, we rely on a recursive inequality involving $S_{\beta_k}(w_{k+1})$ which appears with slight variations in \citep{locatello_stochastic_2019,vladarean2020conditional}. A generic version of this lemma is presented below.

\begin{lemma}
\label{lem:smooth-gap-rec}
For both variants of H-SAG-CGM, and for all $k \geq 1$ it holds that
\begin{align*}
    S_{\beta_k}(w_{k+1})
    & \leq
    (1-\eta_k)S_{\beta_{k-1}}(w_{k}) \\
    & \qquad + \eta_kD_\domain\Expect\|\grad F_\betak(\w_k) - v_k\| \\
    & \qquad + \frac{\eta_k^2D_\domain^2L_{F_{\beta_k}}}{2},
\end{align*}
where $L_{F_{\beta_k}}$ represents the smoothness constant of the surrogate objective $F_{\beta_k}$. If we consider the setting (\hyperlink{setting:v1}{$\texttt{S1}$}) and \hyperref[alg:v1]{Variant~1} of the algorithm, then $L_{F_{\beta_k}} =\frac{\|\X\|L_f}{n} + \frac{\|\A\|}{\betak}$. Otherwise, if $g$ is separable as in (\hyperlink{setting:v2}{$\texttt{S2}$}) and we use \hyperref[alg:v2]{Variant~2}, then $L_{F_{\beta_k}} = \frac{\|\X\|L_f}{n} + \frac{\|\A\|}{\betak m}$. 
\end{lemma}

See \Cref{app:general_bound} for the proof.

\paragraph{Discussion.}
\Cref{lem:smooth-gap-rec} shows how the convergence rate depends on the variance of the stochastic gradient estimator and the design parameters $\eta_k$ and $\beta_k$. Since we can choose $\eta_k$ and $\beta_k$ to get the best possible rates in the analysis, this leaves the variance of the stochastic gradient estimator as the decisive term. By combining this lemma with two gradient estimators, corresponding to \hyperref[alg:v1]{Variants~1} and \hyperref[alg:v2]{2} of \Cref{alg:main}, we get convergence rates on $S_{\beta_{k}}$ which we present in \Cref{main:smoothed_gap_convergence}.

Existing stochastic homotopy CGMs~\citep{locatello_stochastic_2019,vladarean2020conditional} rely on variance-reduced gradient estimators devised to handle arbitrary stochastic objectives, thus failing to exploit the separable finite-sum structure often encountered in practice.

Recently, \citet{negiar2020stochastic} showed that optimal convergence guarantees can be obtained for separable objectives by considering a SAG-like gradient estimator \citep{schmidt2017minimizing}. By combining this idea with the homotopy framework, we are able to provide an improved randomized algorithm (in two variants) for composite objectives.
We now proceed by defining the SAG estimators and presenting their useful properties. %

\subsection{Stochastic Average Gradient (SAG) Error Bounds}

The following two SAG estimators approximate the two parts of the gradient of $F_\beta$. 

At each iteration of \Cref{alg:main}, the $j$-th coordinate of the gradient of $f$ is updated using a SAG estimator~(lines 5-7):
\begin{equation}\label{main:alpha_update_rule}
\begin{aligned}
\alpha_{k,i} & = \begin{cases}
\frac{1}{n} f'(x_i^T w_k) & i = j,\\
\alpha_{k-1, i} & i \neq j,
\end{cases} 
\end{aligned}
\end{equation}

In particular, if we consider setting (\hyperlink{setting:v2}{$\texttt{S2}$}) with a separable $g$, then we can use \hyperref[alg:v2]{Variant 2} of the algorithm which employs another SAG estimator and updates the $l$-th coordinate of the gradient of $g_{\beta_k}$:
\begin{equation}\label{main:gamma_update_rule}
\begin{aligned}
\gamma_{k,q} & = \begin{cases} \frac{1}{m} g_{\beta_k,l}'(\va_l^T\w_k) & q = l,\\ \gamma_{k-1,q} & q\ne l. \end{cases}
\end{aligned}
\end{equation}

Otherwise, in setting (\hyperlink{setting:v1}{$\texttt{S1}$}) with a non-separable non-smooth $g$, we use full gradients of $g_{\beta_k}$ as in \hyperref[alg:v1]{Variant~1}.

In summary, \hyperref[alg:v1]{Variant~1} assumes stochastic~$\nabla f$ approximated by $\alpha_k$ and a non-separable $g$ whose gradient is fully computed. Thus, the stochastic gradient is a sum of a stochastic and deterministic terms: $v_k = X^T\alpha_k + A^T\nabla g_{\beta_k}(Aw_k)$.

On the other hand, 
\hyperref[alg:v2]{Variant~2} assumes that $g$ separable in addition to $f$, hence $\nabla g_{\beta_k}$ can be approximated by $\gamma_k$. Thus, the overall gradient term $v_k$ in this case is the sum of two stochastic terms given by $v_k = X^T\alpha_k + A^T\gamma_k$.

We now present two lemmas characterizing the errors of $\alpha_k$ and $\gamma_k$ in $\ell_1$-norm.

\begin{lemma}\label{main:alpha_err}
[Lemma~3 in \citep{negiar2020stochastic}] \\
Consider H-SAG-CGM, with the SAG estimator $\alpha_{k}$ defined in \eqref{main:alpha_update_rule}. Then, for all $k \geq 2$, 
\begin{align*}
    \mathbb E \left[\|\grad f(Xw_k) - \alpha_k\|_1\right]
    & \leq
    \left(1 - \tfrac{1}{n}\right)^{k} \|\grad f(Xw_0) - \alpha_0\|_1 \\
     & \quad + C_1 \left(1 - \tfrac{1}{n}\right)^{k/2}  \log k
    +\frac{C_2}{k},
\end{align*}
where $C_1 = 2n^{-1} L_f D_1(X)$, $C_2 = 4n^{-1}(n-1) L_f D_1(X)$ and the expectation is taken over all previous steps in the algorithm. 
\end{lemma}

\begin{lemma}
\label{lem-gamma-error}
Consider \hyperref[alg:v2]{Variant~2} of H-SAG-CGM with the SAG estimators defined in \eqref{main:alpha_update_rule} and \eqref{main:gamma_update_rule}. Then, for all $k\geq 2$,
\begin{align*}
&\Expect[\|\nabla g_{\beta_{k}}(Aw_k) - \gamma_k\|_1] \\
&\qquad\leq \left( 1 - \tfrac{1}{m} \right)^k \|\nabla g_{\beta_{0}}(Aw_0) - \gamma_0\|_1
+\frac{C}{\sqrt{k}}
\end{align*}
where $C = 10 \beta_0^{-1}D_1(A)$ and the expectation is taken over all previous steps of the algorithm. 
\end{lemma}

We refer to \citep{negiar2020stochastic} for the proof of \Cref{main:alpha_err}. We present the proof of \Cref{lem-gamma-error} in \Cref{app:lem-gamma-error-proof} under the assumption that $g$ is an indicator function or a Lipschitz continuous function.

\paragraph{Discussion.}

\Cref{main:alpha_err} shows that the SAG-like estimator provides an error bound in $\ell_1$-norm that decays as $\bigO(1/k)$ in expectation.
This decay does not carry over to the separable case in \hyperref[alg:v2]{Variant~2}, as demonstrated by \Cref{lem-gamma-error}, due to the $\frac{1}{\beta_k}$-factor associated with the smoothed approximation $g_{\beta_k}$. %

\subsection{Convergence Rates}
\label{sec:convergence-rates}

Combining \Cref{main:alpha_err,lem-gamma-error} with \Cref{lem:smooth-gap-rec} gives the convergence rates for the two variants of H-SAG-CGM which we now present.

\begin{thm}\label{main:smoothed_gap_convergence}
The sequence generated by H-SAG-CGM (\Cref{alg:main}) satisfies, for all $k\geq 2$,
\begin{align}
S_{\beta_k}(w_{k+1})
\leq
\frac{C_1}{\sqrt{k}}
+ \frac{C_2}{k}
+ \frac{C_3}{k^2}.\notag
\end{align}
The constants are defined for H-SAG-CGM/v1 as follows:
\begin{itemize}[$\triangleright$,topsep=-1pt,itemsep=-1ex,partopsep=1ex,parsep=1ex]
\item $C_1 = 2 D_\domain^2 \|A\| \beta_0^{-1}$
\item $C_2 = 8 L_f D_1(X) D_{\infty}(X) + 2 n^{-1} L_f \|X\| D_\domain^2$
\item $C_3 = 2 n^2 D_{\infty}(X) \big( \|\grad f(Xw_1) - \alpha_0\|_1  + 32 L_f D_1(X) \big)$
\end{itemize}
and for \hyperref[alg:v2]{Variant 2} as follows:
\begin{itemize}[$\triangleright$,topsep=-1pt,itemsep=-1ex,partopsep=1ex,parsep=1ex]
    \item $C_1 = \beta_0^{-1}(2 D_\domain^2 \|A\| + 10 D_1(A))$.
    \item $C_2 = 8 L_f D_1(X) D_{\infty}(X) + 2 n^{-1} L_f \|X\| D_\domain^2$
    \item $C_3 = 2 n^2 D_{\infty}(X) \big( \|\grad f(Xw_1) - \alpha_0\|_1  + 32 L_f D_1(X) \big) + 2 m^2 D_\infty(A)\|\nabla g_{\beta_0}(Aw_1) - \gamma_0\|_1$
\end{itemize}
\end{thm}

Using the techniques described in \citep{trandinh2018smooth}, we translate this bound to convergence guarantees on the original problem in the following corollaries.

\begin{corollary}\label{main:final_convergence_Lipschitz}
Suppose $g:\R^m \to \R$ is $L_g$-Lipschitz continuous. Then, the estimates generated by H-SAG-CGM (\Cref{alg:main}) satisfy 
\begin{align}
    \Expect[F(w_{k+1}) - \optimal F] \leq \frac{C_1}{\sqrt{k}}
+ \frac{C_2}{k}
+ \frac{C_3}{k^2}\notag
 + \frac{\beta_0 L_g^2}{2\sqrt{k}}
\end{align}
where the constants $C_1, C_2$ and $C_2$ are defined in \Cref{main:smoothed_gap_convergence}.

\end{corollary}

\begin{corollary}\label{main:final_convergence}
Suppose $g$ is the indicator function of a closed and convex set $\mathcal{K}$.
Then, for H-SAG-CGM (\Cref{alg:main}), we have a lower bound on the suboptimality as $\Expect \left[ f(Xw_{k+1}) - f(Xw^*)\right] \geq  -\|y^*\|\Expect\left[\dist(Aw_{k+1}, \constset)\right]$ and the following upper bounds on the suboptimality and feasibility:
\begin{align*}
    \Expect\left[ f(Xw_{k+1}) - f(Xw^*)\right] & \leq \frac{C_1 + \beta_0}{\sqrt{k}} + \frac{C_2}{k}+ \frac{C_3}{k^2},\enspace\text{and}\enspace \\
    \Expect\left[\dist(Aw_{k+1}, \constset)\right] 
    & \leq \frac{C_4}{\sqrt{k}}  
    + \frac{\sqrt{2C_2}}{k^{3/4}} + \frac{\sqrt{2C_3}}{k^{5/4}}
\end{align*}
where the constants $C_1, C_2$ and $C_3$ are defined in \Cref{main:smoothed_gap_convergence} and ${C_4 = (\frac{3\beta_0\|y^*\|}{2} + \sqrt{2C_1})}$.
\end{corollary}

\paragraph{Discussion.}
Even in the deterministic setting studied in \citep{yurtsever2018conditional}, the convergence rates of Homotopy CGM is bounded below by $\Omega(1/\sqrt{k})$, as demonstrated theoretically in \citep{lan2013complexity} and practically in \citep{Kerdreux2021LocalAG}. \Cref{main:final_convergence_Lipschitz,main:final_convergence} show that both variants of our algorithm achieves this lower bound.

H-SAG-CGM/v1
provides an order of magnitude improvement (from $\bigO(\varepsilon^{-3})$ to $\bigO(\varepsilon^{-2})$) over the previous state-of-the-art in deterministic constraints, \citet{locatello_stochastic_2019}.

While H-SAG-CGM/v2 and H-SPIDER-FW \citet{vladarean2020conditional} enjoy a similar overall rate, the latter requires an exponentially increasing batch size. Combined with occasional full passes, this quickly becomes impractical for strongly constrained problems. As an alternative, \citet{vladarean2020conditional} propose H-1SFW which does use a fixed batch size but at the cost of an impractical~$\bigO (\varepsilon^{-6})$ rate. In stark contrast, our algorithm enjoys the optimal rate without resorting to increasing the batch size.

\section{NUMERICAL EXPERIMENTS}

\begin{figure*}[!t]
    \centering
    \includegraphics[width=\linewidth]{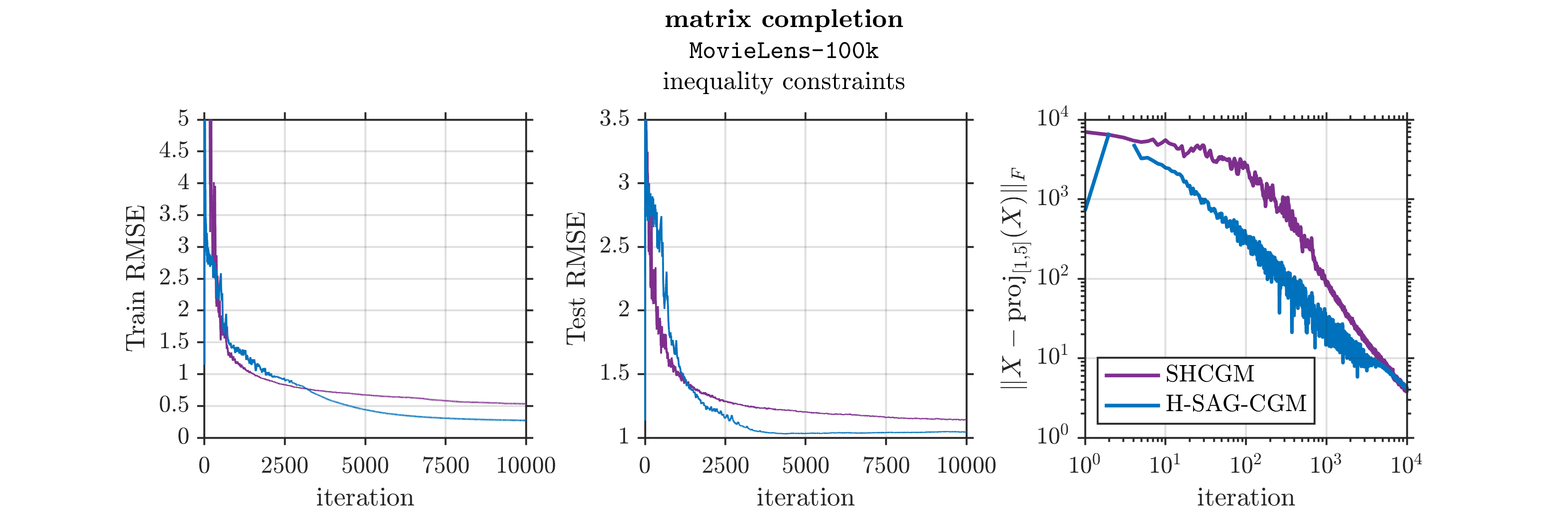}
    \caption{Empirical comparison of H-SAG-CGM/v1 with SHCGM on matrix completion with inequality constraints~\eqref{eq:matrix-completion} with the \texttt{MovieLens-100k} dataset.}
    \label{fig:matrix-completion-experiment}
\end{figure*}

\begin{figure*}[!t]
    \centering
    \includegraphics[width=\linewidth]{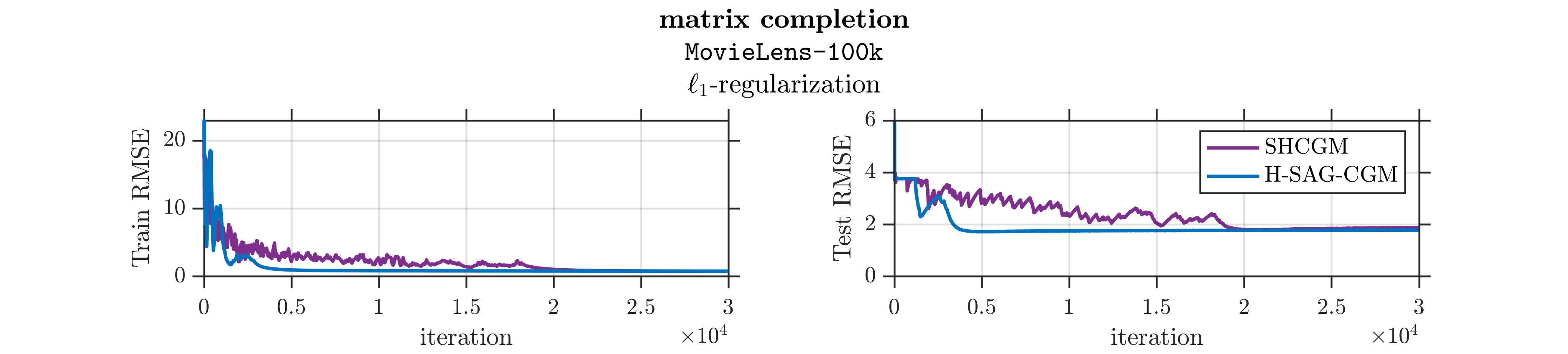}
    \caption{Empirical comparison of H-SAG-CGM/v1 with SHCGM on matrix completion with $\ell_1$-regularization~\eqref{eq:matrix-completion-l1} with the \texttt{MovieLens-100k} dataset.}
    \label{fig:matrix-completion-experiment-l1}
\end{figure*}

\begin{figure*}[!t]
    \centering
    \includegraphics[width=\linewidth]{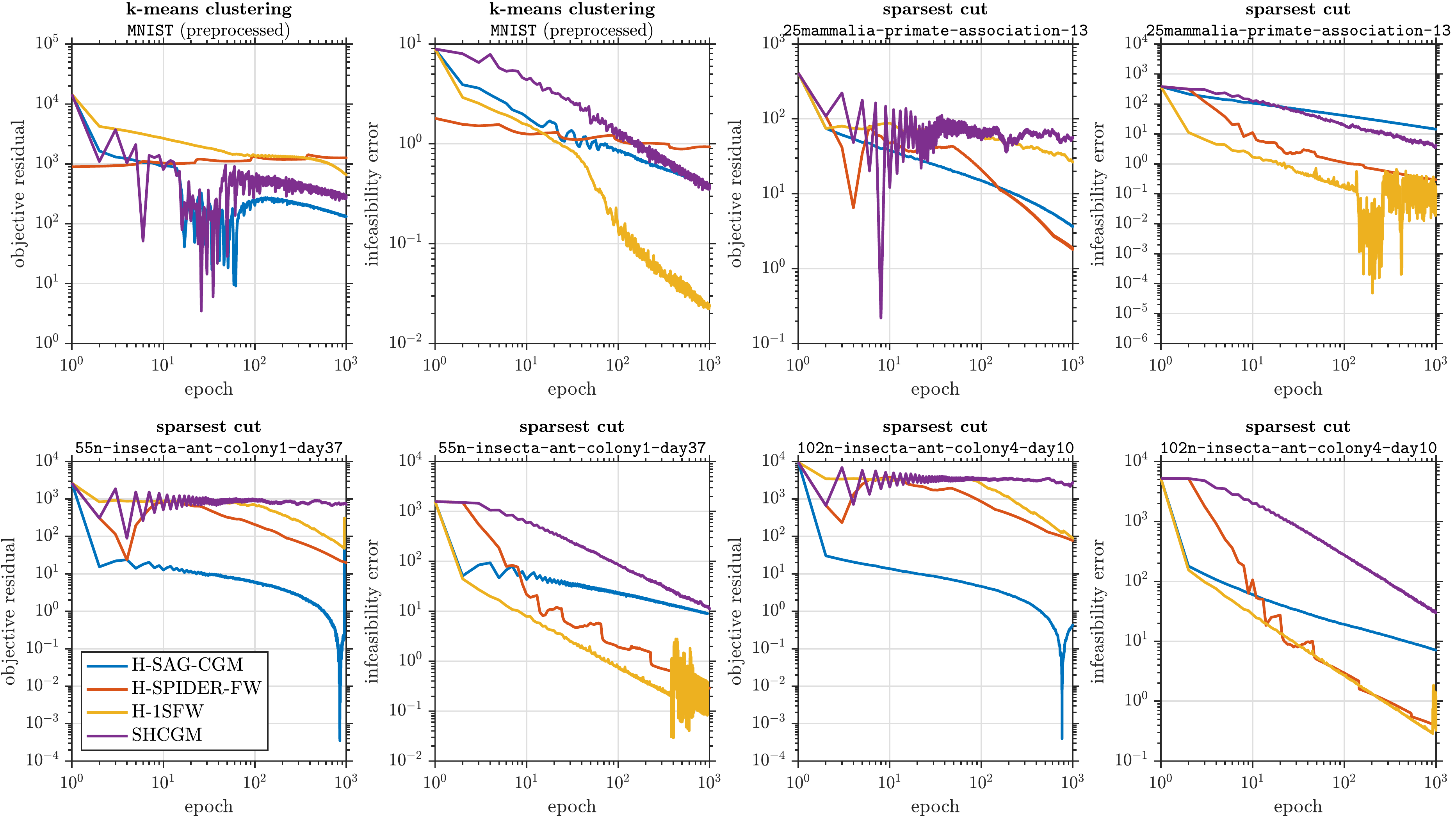}
    \caption{Comparing H-SAG-CGM/v2 to state-of-the-art baselines on two distinct SDP-relaxation tasks, \mbox{$k$-means}~\eqref{eq:kmeans} and sparsest cut \eqref{eq:sparsestcut}. The x-scale is in terms of the constraint epochs. One constraint epoch corresponds to a full pass over all the constraints. Note that for the $k$-means clustering experiment, we deliberately restricted H-SPIDER-FW to not perform full passes over all of the constraints resulting in noticeable degradation in performance.}
    \label{fig:sdp-experiments}
\end{figure*}
 
This section demonstrates the empirical performance of the proposed method across a number of different problems: matrix completion, k-means clustering and uniform sparsest cut. We performed these experiments in MATLAB R2019b and the codes are publicly available at \url{https://github.com/ratschlab/faster-hcgm-composite}.

\paragraph{Baselines.}
We compare the proposed method against the following methods: \\[0.25em]
$\triangleright$~SHCGM \citep{locatello_stochastic_2019}\\
$\triangleright$~H-SPIDER-FW \citep{vladarean2020conditional}\\
$\triangleright$~H-1SFW \citep{vladarean2020conditional}

Note that SHCGM only works in the case of deterministic $g$ and is hence a natural baseline comparison for H-SAG-CGM/v1. H-SPIDER-FW can handle stochastic $g$ so it is used to compare to H-SAG-CBM/v2 but importantly in this case, H-SPIDER-FW requires an increasing batch size.

\paragraph{Challenges.}
The parameter $\beta$ determines a trade-off between convergence in the objective residual and the infeasibility error. However, since we do not know the optimal value a priori, $\beta_0$ is not always easy to interpret given a particular task. This leaves practitioners to develop the intuition on how to tune this parameter. This challenge is not unique to H-SAG-CGM but is shared among homotopy CGM approaches \citep{yurtsever2018conditional,locatello_stochastic_2019,vladarean2020conditional}.
Automating $\beta_0$-tuning is an important direction for future research.

\subsection{Matrix Completion}

We consider two different formulations of the matrix completion problem. First, we focus on matrix completion with hard inequality constraints studied in \citet{locatello_stochastic_2019}:
\begin{equation}\label{eq:matrix-completion}
\hspace{-3mm}
\min_{\|\w\|_\star\leq \zeta}
\sum_{(i, j) \in \Omega}\left(\w_{ij}-X_{ij}\right)^{2}\; \text{subject to} \; 1 \leq \w \leq 5
\end{equation}
where $\Omega$ is the observed entries of the input data $X$, and~$\|X\|_\star$~denotes the nuclear norm. The inequality constraints $1\leq\w\leq 5$ are hard thresholds which specify that all the entries of $\w$ must lie between $1$ and $5$.

For $X$, we used the \textsc{Movielens-100k} dataset\footnote{F.M.\ Harper, J.A.\ Konstan.\ --- Available at https://grouplens.org/datasets/movielens} containing approximately 100,000 integer valued movie ratings between $1$ and $5$, assigned by 1682 users to 943 movies. We used the \texttt{ub.train} and \texttt{ub.test} partitions provided with the original data for the train/test split.

This numerical setup was studied also in \citet{locatello_stochastic_2019}. We used the parameter setting that they reported without any further tuning. We set ${\zeta=\texttt{7e3}}$ for the nuclear norm bound, ${\beta_0 = \texttt{10}}$ for the initial smoothing parameter, and we compute gradient estimators with $\texttt{1000}$ \textit{iid} samples at each iteration.

\Cref{fig:matrix-completion-experiment} compares the performance of H-SAG-CGM/v1 against SHCGM \citep{locatello_stochastic_2019} in terms of train and test root mean squared error (RMSE) and infeasibility error. The comparison is based on the iteration counter, which is an arguably fair representation of the time cost of the algorithms since both methods use the same number of samples per iteration.

Next, we test our algorithm for a setting in which $g$ is Lipschitz continuous by performing experiments on matrix completion with $\ell_1$-regularization:
\begin{equation}\label{eq:matrix-completion-l1}
\min_{\|\w\|_\star\leq \zeta}
\sum_{(i, j) \in \Omega}\left(\w_{ij}-X_{ij}\right)^{2}
+ \lambda \|w\|_1.
\end{equation}
We use the same dataset and parameter settings as in~\eqref{eq:matrix-completion} with the regularization parameter set to $\lambda = \texttt{0.1}$. 
\Cref{fig:matrix-completion-experiment-l1} presents the train and test RMSE obtained in this experiment. Note that the estimates remain feasible in this experiment since $g$ is not an indicator function.

\subsection{\texorpdfstring{$k$}{TEXT}-Means Clustering}\label{sec:clustering}

In this experiment, we test H-SAG-CGM/v2. The goal in $k$-means is to assign $n$ data points to $k$ clusters. We consider the following SDP relaxation of this problem~\citep{peng2007approximating}:
\begin{equation}\label{eq:kmeans}
\begin{array}{ll}
\displaystyle\min_{\w\in\mathcal X}
& \langle \w, \C\rangle \quad\text {subject to } \w\vec{1} = \vec{1}, \text{ and }  w\geq 0
\end{array}
\end{equation}
where $\mathcal X= \{\w\in\mathbb S^{n\times n}_+ \mid \Tr(\w)\leq\frac{1}{n}\}$, $\vec{1} = [1,1,\ldots,1]\in\mathbb R^n$, and $\w\geq 0$ denotes entry-wise non-negativity. 
The problem is strongly constrained with a total of $n^2+n$ constraints --- $n$ equality and $n^2$~inequality constraints.

This problem is also studied in the related works on homotopy CGM in \citep{yurtsever2018conditional, locatello_stochastic_2019, vladarean2020conditional}. We use the same test setup: For the input data $C$, we use \textsc{mnist} dataset\footnote{\url{http://yann.lecun.com/exdb/mnist/}} with the preprocessing considered in \citet{mixon2016clustering}. We set $\beta_0 = \texttt{7}$.

We compare the methods based on the number of epochs (an epoch corresponds to a full pass over the constraints) since different methods use different batch sizes in this experiment. The first two plots in \Cref{fig:sdp-experiments} present the outcomes of this experiment. We measure the objective residual as relative suboptimality $|f(\w_k) - \optimal f| / |\optimal f|$ and the infeasibility error as the Euclidean distance to the feasible set $\mathrm{dist(Aw_{k},\mathcal{K})}$.

\subsection{Uniform Sparsest Cut}\label{sec:sparsestcut}

In this experiment, we test H-SAG-CGM/v2 on the uniform sparsest cut SDP. This problem is particularly interesting because of the $\bigO(n^3)$ number of constraints.

Let $G = (V,E)$ be a graph with $n$ nodes $|V| = n$ and a set of edges $E$. The goal in uniform sparsest cut is to split vertices into two partitions $(S,\bar S)$ that minimize
\begin{equation}
    \frac{|E(S,\bar S)|}{|S||\bar S|}
\end{equation}
where $E(S,\bar S)\subseteq E$ is the set of edges between the nodes in $S$ and $\bar S$.

This canonical problem has applications across many fields including VLSI circuit layout design, the topological design of communication networks, image segmentation, and many others. In machine learning, it is a sub-problem of hierarchical clustering \citep{dasgupta2016cost,chatziafratis2018hierarchical}. %

\citet{arora2009expander} propose a $\bigO(\sqrt{\log n})$-approximation algorithm for this problem based on an SDP relaxation with $\mathcal O(n^3)$ triangle inequality constraints. We adapt their formulation to our SDP model \eqref{main:basic_sdp_def}:
\begin{equation}\label{eq:sparsestcut}
\begin{array}{ll}
\displaystyle\min_{\substack{w\in \mathbb{S}^{n\times n}_+\\ \Tr(w)\leq n}} & \langle L,w\rangle\\[1em]
\text{subj.\ to} & n\Tr(w) - \Tr(\mathds 1_{n\times n}w) = \displaystyle\frac{n^2}{2}\\[1em]
& w_{ij} + w_{jk} - w_{ik} - w_{jj} \leq 0\quad\forall i,j,k\in V
\end{array}
\end{equation}
where $L$ is the graph Laplacian of $G$.

We used three datasets from the Network Repository \citep{rossi2015network}:\footnote{\url{https://networkrepository.com}} \textsc{25mammalia-primate-associate-13}, \textsc{55n-insecta-ant-colony1-day37}, and \textsc{102n-insecta-ant-colony4-day10}. These three datasets differ in size by a factor of ten. See \Cref{tab:usc_graphs} in the Appendix for more details. We use $\beta_0 = \texttt{100}$ for all three network datasets. 

\Cref{fig:sdp-experiments} presents the results of this experiment.  As in the k-means experiment, the objective residual infeasibility error represent $|f(\w_k) - \optimal f| / |\optimal f|$ and $\mathrm{dist(Aw_{k},\mathcal{K})}$ respectively. H-SPIDER-FW is affected by the growing number of constraints because of its increasing batch size strategy. Other methods, with constant batch size, are less affected. H-SAG-CGM/v2 performs competitively against H-SPIDER-FW without requiring an increasing batch size.  %

\section{CONCLUSION}

We developed a fast randomized conditional gradient method for solving convex composite finite-sum problems. The proposed method is particularly suitable for solving SDPs with a large number of affine constraints. Theoretically, the proposed method has favorable scaling properties compared to the previous state-of-the-art. Empirically, it performs on par with more sophisticated variance reduction techniques. 

The proposed method takes advantage of a structural assumption on the separability of the objective by applying randomization. For the non-smooth term, the proposed method tackles the two subcases of deterministic and stochastic separately. If the non-smooth term is deterministic, the proposed method obtains an $\varepsilon$-suboptimal solution after $\mathcal{O}(\varepsilon^{-2}dm)$ arithmetic operations (where $d$ is the dimensionality of the decision variable and $m$ is the number of constraints comprising~$g$). This improves the previous complexity of $\mathcal O(\varepsilon^{-3}dm)$ found in \citet{locatello_stochastic_2019}. 

If we further assume that the non-smooth part is also separable, then we can employ a fully randomized scheme to find an $\varepsilon$-suboptimal solution after $\mathcal{O}(\varepsilon^{-2}d)$ arithmetic operations. This total cost complexity is independent of $m$ and thus represents a significant improvement compared over previous work \citep{vladarean2020conditional} which has a total cost of $\mathcal{O}(\varepsilon^{-2}dm)$.

\subsubsection*{Acknowledgments}

The authors thank Vincent Fortuin for his helpful feedback on an initial draft of this work and the anonymous reviewers for their detailed comments. This work started while F.L. was at ETH Z\"urich and is based on the research done outside of Amazon.

This  work  was  supported  by ETH core funding to G.R. (funding G.D.). 
V.C.\ has received funding from the European Research Council (ERC) under the European Union's Horizon 2020 research and innovation programme (grant agreement n$^\circ$ 725594 -- time-data). 
M.L.V.\ was supported by the Swiss National Science Foundation (SNSF) for the project ``Theory and Methods for Storage-Optimal Optimization'' grant number 200021\_178865. 
A.Y.\ received support from the Wallenberg AI, Autonomous Systems and Software Program (WASP) funded by the Knut and Alice Wallenberg Foundation.

\bibliography{references}
\bibliographystyle{plainnat}

\appendix

\onecolumn \makesupplementtitle
\setcounter{table}{0}
\setcounter{figure}{0}
\renewcommand{\thetable}{S\arabic{table}}  
\renewcommand{\thefigure}{S\arabic{figure}}

\allowdisplaybreaks %

\section{Background on Smoothing}

This section recalls some useful properties about the smoothing technique \citep{nesterov05}. We present these known properties in this section for completeness, since we use them in our analysis. 

Let $g:\R^m \to \R \cap \{+\infty\}$ be a proper, closed and convex function. The smooth approximation of $g$ is defined by
\begin{equation}\label{eqn:supp-smoothing}
g_{\beta} (z) = \max_{y \in \R^d} \left\{ \ip{z}{y} - g^\ast(y) - \frac{\beta}{2} \norm{y}^2 \right\}
\end{equation}
where $g^\ast$ denotes the Fenchel conjugate and $\beta > 0$ is the smoothing parameter. 
Then, $g_\beta$ is convex and $\tfrac{1}{\beta}$-smooth. 
Let $y^\ast_\beta(z)$ denote the solution of the maximization sub-problem in \eqref{eqn:supp-smoothing}, i.e.,
\begin{align}
y^\ast_\beta(z) & = \arg\max_{y \in \R^d}  \left\{ \ip{z}{y} - g^\ast(y) - \frac{\beta}{2} \norm{y}^2 \right\} \label{eqn:supp-ybeta-defn}\\
& = \arg\min_{y \in \R^d} \left\{\frac{1}{\beta} g^\ast(y) - \frac{1}{\beta} \ip{z}{y} + \frac{1}{2} \norm{y}^2 + \frac{1}{2} \norm{\frac{1}{\beta} z}^2 \right\} \\
& = \arg\min_{y \in \R^d} \left\{ \frac{1}{\beta} g^\ast(y) + \frac{1}{2} \norm{y - \frac{1}{\beta} z}^2 \right\} \\
& = \prox_{\beta^{-1} g^\ast} (\beta^{-1} z) \\
& = \frac{1}{\beta} \big( z - \prox_{\beta g} (z) \big)
\end{align}
where the last line is the Moreau decomposition. 
Then, the followings hold $\forall z_1, z_2 \in \R^m$ and $\forall \beta, \gamma > 0$
\begin{align}
g_\beta(z_1) & \geq g_\beta(z_2) + \ip{\nabla g_{\beta}(z_2)}{z_1 - z_2} + \frac{\beta}{2} \norm{y^\ast_\beta(z_2) - y^\ast_\beta(z_1)}^2 \label{eqn:smoothing-prop-1} \\
g(z_1) & \geq g_\beta(z_2) + \ip{\nabla g_{\beta}(z_2)}{z_1 - z_2} + \frac{\beta}{2} \norm{y^\ast_\beta(z_2)}^2 \label{eqn:smoothing-prop-2} \\
g_{\beta} (z_1) & \leq g_{\gamma}(z_1) + \frac{\gamma - \beta}{2} \norm{y^\ast_\beta(z_1)}^2 \label{eqn:smoothing-prop-3}
\end{align}
We refer to Lemma~10 in \citep{trandinh2018smooth} for the proofs. 

Suppose that $g$ is $L_g$-Lipschitz continuous. Then, for $\forall \beta > 0$ and $\forall z \in \R^m$, 
\begin{equation}\label{eqn:smoothing-sandwich}
g_{\beta} (z) \leq g (z) \leq g_{\beta} (z) + \frac{\beta}{2} L_g^2,
\end{equation}
The proof follows immediately from Equation~(2.7) in \citep{nesterov05} with a remark on the duality between bounded domain and Lipschitz continuity. %

\clearpage 

\section{Proof of \texorpdfstring{\Cref{lem:smooth-gap-rec}}{Lemma~\ref{lem:smooth-gap-rec}}}\label{app:general_bound}

We follow the steps laid out in Theorem 4.1 in \citep{vladarean2020conditional}, which in turn builds upon Theorem 9 in \citep{locatello_stochastic_2019}. 

We use the  quadratic upper bound ensured by the fact that $F_\betak$ is $L_{F_\betak}$-smooth:
\begin{align}
    F_\betak(w_{k+1})
    &\leq
    F_\betak(w_k) + \innerp{\grad F_\betak(w_k), w_{k+1} - w_k} + \frac{L_{F_\betak}}{2}\|w_{k+1} - w_k\|^2\\
    &\leq
    F_\betak(w_k) + \eta_k\innerp{\grad F_\betak(w_k), s_k - w_k} + \frac{\eta_k^2L_{F_\betak}D_\domain^2}{2}
\end{align}
where the second line follows from the boundedness of $\domain$. 

Next, we use the rule for change of $\beta$ in smoothing (see \eqref{eqn:smoothing-prop-3}), which gives 
\begin{align}
    F_\betak(w_{k+1})
    &\leq F_\betakminusone(w_k) + \frac{\betakminusone-\betak}{2} \|\optimal y_\betak(Aw_k)\|^2 + \eta_k\innerp{\grad F_\betak(w_k), s_k - w_k} + \frac{\eta_k^2L_{F_\betak}D_\domain^2}{2},\label {eq:qub_unexpanded_gap}
\end{align}
where $y_{\beta_k}^\star$ is defined as in \eqref{eqn:supp-ybeta-defn}. 

Then, we bound the term $\innerp{\grad F_\betak(w_k), s_k - w_k}$ as follows:
\begin{align}
\innerp{\grad F_\betak(w_k), & s_k - w_k} =
\innerp{\grad F_\betak(w_k) - v_k, s_k - w_k} + \innerp{v_k, s_k - w_k}\\
&= \innerp{\grad F_\betak(w_k) - v_k, s_k - \optimal{w}} + \innerp{\grad F_\betak(w_k) - v_k, \optimal{w} - w_k} + \innerp{v_k, s_k - w_k}\\
&\leq
\innerp{\grad F_\betak(w_k) - v_k, s_k - \optimal{w}} + \innerp{\grad F_\betak(w_k) - v_k, \optimal{w} - w_k} + \innerp{v_k, \optimal{w} - w_k}\\
&=
\innerp{\grad F_\betak(w_k) - v_k, s_k - \optimal{w}} + \innerp{\grad F_\betak(w_k), \optimal{w} - w_k}%
\end{align}
where the inequality follows by the definition of $s_k$. 

Now, we focus on the term $\innerp{\grad F_\betak(w_k), \optimal{w} - w_k}$ and bound it as follows:
\begin{align}
\innerp{\grad F_\betak(w_k), \optimal{w} - w_k}
&=  \innerp{X^T\grad f(Xw_k) + A^T \grad g_\betak(Aw_k), \optimal{w} - w_k}\\
&=  \innerp{\grad f(Xw_k), X(\optimal{w} - w_k)} + \innerp{\grad g_\betak(Aw_k), A(\optimal{w} - w_k)}\\
&\leq  f(X\optimal{w}) - f(Xw_k) + g(A\optimal{w}) - g_\betak(Aw_k) 
-\frac{\betak}{2}\|\optimal{y}_\betak(Aw_k)\|^2 \label{eq:apply_tran2018} \\
&= \optimal F - F_\betak(w_k)
-\frac{\betak}{2}\|\optimal{y}_\betak(Aw_k)\|^2,
\end{align}
where the inequality holds due to the convexity of $f$ and $g$ and the smoothing property in \eqref{eqn:smoothing-prop-2}. 

Combining all these bounds and subtracting $\optimal F$ from both sides, we get
\begin{align}
    F_\betak(w_{k+1}) - \optimal F &\leq
    (1-\eta_k) \left( F_\betakminusone(w_k) - \optimal F\right)
    + \eta_k \innerp{\grad F_\betak(w_k) - v_k, s_k - \optimal{w}} \notag \\
    & \qquad
    + \frac{1}{2}((1-\eta_k)(\betakminusone-\betak) - \eta_k\betak)\elltwo{\optimal y_\betak(Aw_k)}^2
   + \frac{\eta_k^2L_{F_\betak}D_\domain^2}{2}
\end{align}

 We cannot bound $\elltwo{\optimal y_\betak(Aw_k)}^2$ in general, so we choose $\eta_k$ and $\beta_k$ carefully to vanish this term. Let $\eta_k = \frac{2}{k+1}$ and $\beta_k = \frac{\beta_0}{\sqrt{k+1}}$ for an arbitrary $\beta_0>0$. Then,
\begin{align}
    (1-\eta_k)(\betakminusone-\betak) - \eta_k\betak 
   = \frac{\beta_0}{\sqrt{k}}\left( \frac{k-1}{k+1} - \frac{\sqrt{k}}{\sqrt{k+1}} \right)
     < 0,
     \qquad \text{for all $k\geq1$.}
\end{align}

Finally, taking expectation on both sides and applying the definition of $S_{\beta}(w) \defeq \Expect \left[F_\beta(w) - \optimal F\right]$ we arrive at our stated result:
\begin{align}
    S_\betak(w_{k+1})&\leq
    (1-\eta_k)S_\betakminusone(w_k)
    + \eta_k \Expect[\innerp{\grad F_\betak(w_k) - v_k, s_k - \optimal{w}}]
   + \frac{\eta_k^2L_{F_\betak}D_\domain^2}{2}.
\end{align}

\clearpage

\section{Proof of \texorpdfstring{\Cref{main:smoothed_gap_convergence}}{Theorem~\ref{main:smoothed_gap_convergence}}}\label{app:conv_rate_det_constraints}

Our aim is to get a rate on the smoothed gap $S_\betak(w_{k+1})$. 
We start from \Cref{lem:smooth-gap-rec}:
\begin{align}
    S_\betak(w_{k+1})&\leq
    (1-\eta_k)S_\betakminusone(w_k)
    + \eta_k \Expect[\innerp{\grad F_\betak(w_k) - v_k, s_k - \optimal{w}}]
   + \frac{\eta_k^2}{2}\left(\frac{\|X\|L_f}{n} + \frac{\|A\|}{\betak }\right)D_\domain^2. \label{eq:smoothed-gap-bound-1}
\end{align}
Multiply both sides by $k(k+1)$ and unroll the recurrence to get
\begin{align}
    k(k+1) S_\betak(w_{k+1})
    & \leq (k-1)kS_\betakminusone(w_k) + 2 k \Expect[\innerp{\grad F_\betak(w_k) - v_k, s_k - \optimal{w}}] + \frac{2 k}{k+1}\left(\frac{\|X\|L_f}{n} + \frac{\|A\|}{\betak }\right)D_\domain^2 \notag \\
    & \leq \underbrace{\sum_{i=1}^k 2 i \Expect[\innerp{\grad F_{\beta_i}(w_i) - v_i, s_i - \optimal{w}}]}_{\circled{A}} +  \underbrace{\sum_{i=1}^k \frac{2 i}{i+1}\left(\frac{\|X\|L_f}{n} + \frac{\|A\|}{\beta_i }\right)D_\domain^2}_{\circled{B}}. \label{eqn:supp-thm41-ABbound}
\end{align}

First, we get an upper-bound on the variance term {\smash{\circled{A}}} as follows: 
\begin{align}
\Expect[\innerp{\grad F_\betak(w_k) - v_k, s_k - \optimal{w}}] 
& = \Expect[\innerp{X^T(\grad f(Xw_k) - \alpha_k) + A^T(\grad g_\betak(Aw_k) - \gamma_k), s_k - \optimal{w}}] \\
& = \Expect[\innerp{\grad f(Xw_k) - \alpha_k, X(s_k - \optimal{w})}] \\
& \leq \Expect[\norm{\grad f(Xw_k) - \alpha_k}_1 \,\norm{X(s_k - \optimal{w})}_\infty] \\
& \leq \Expect[\norm{\grad f(Xw_k) - \alpha_k}_1] \, D_{\infty}(X) \label{eqn:supp-thm41-variance-1}
\end{align}

where, the first inequality is the H\"older's inequality, and the second one is based on the boundedness of $\domain$.

Then, by \Cref{main:alpha_err}, we have  
\begin{align}
\Expect[\|\grad f(Xw_k) - \alpha_k\|_1] 
\leq \left(1 - \tfrac{1}{n}\right)^{k} \|\grad f(Xw_1) - \alpha_0\|_1 + \frac{2 L_f D_1(X)}{n} \left(  \left(1 - \tfrac{1}{n}\right)^{k/2}  \log k + \frac{2(n-1)}{k} \right).\label{eqn:supp-thm41-variance-2}
\end{align}
Finally, we combine \eqref{eqn:supp-thm41-variance-1} and \eqref{eqn:supp-thm41-variance-2} to get
\begin{align}
\circled{A} 
& \leq 2 D_{\infty}(X) \left[ \|\grad f(Xw_1) - \alpha_0\|_1 \sum_{i=1}^k i \left(1 - \tfrac{1}{n}\right)^{i}  + \frac{2 L_f D_1(X)}{n} \sum_{i=1}^k\left( i \left(1 - \tfrac{1}{n}\right)^{i/2}  \log i + 2(n-1) \right) \right] \\
& \leq 2 D_{\infty}(X) \left[ \|\grad f(Xw_1) - \alpha_0\|_1 \, n^2  + \frac{2 L_f D_1(X)}{n} \left( 16 n^3 + 2(n-1)k \right) \right] \\
& \leq 2 D_{\infty}(X) \left[ \|\grad f(Xw_1) - \alpha_0\|_1 \, n^2  + 4 L_f D_1(X) \left( 8 n^2 + k \right) \right]
\end{align}
where we use \Cref{lem:bounds-for-sums} for the second line. 

Next, we focus on the term \smash{\circled{B}}, and we use once again \Cref{lem:bounds-for-sums} and obtain
\begin{align}
\circled{B} 
& = 2 D_\domain^2 \left( \frac{\|X\|L_f}{n}\sum_{i=1}^k \frac{ i}{i+1} +  \frac{\|A\|}{\beta_0} \sum_{i=1}^k \frac{ i}{\sqrt{i+1}} \right) \leq 2 D_\domain^2 \left( \frac{\|X\|L_f}{n} k +  \frac{\|A\|}{\beta_0} k \sqrt{k+1}  \right).
\end{align}

To finalize, we substitute the bounds on \smash{\circled{A}} and \smash{\circled{B}} into \eqref{eqn:supp-thm41-ABbound} and divide both sides by $k(k+1)$ to get the desired bound on $S_\betak (w_{k+1})$:
\begin{align*}
    S_\betak(w_{k+1})
    & \leq \frac{2 D_{\infty}(X)}{k(k+1)}  \left[ \|\grad f(Xw_1) - \alpha_0\|_1 \, n^2  + 4 L_f D_1(X) \left( 8 n^2 + k \right) \right] + \frac{2 D_\domain^2}{k(k+1)}  \left( \frac{\|X\|L_f}{n} k +  \frac{\|A\|}{\beta_0} k \sqrt{k+1} \right)  \\
    & \leq \frac{C_3}{k(k+1)} + \frac{C_2}{k+1} + \frac{C_1}{\sqrt{k+1}}, \quad \text{where} \quad 
    \begin{aligned}[t]
    C_3 & = 2 n^2 D_{\infty}(X) \big( \|\grad f(Xw_1) - \alpha_0\|_1  + 32 L_f D_1(X) \big) \\
    C_2 & = 8 L_f D_1(X) D_{\infty}(X) + 2 n^{-1} L_f \|X\| D_\domain^2  \\
    C_1 & = 2 D_\domain^2 \|A\| \beta_0^{-1}.
    \end{aligned}
\end{align*}

\subsection{Proof of \texorpdfstring{\Cref{main:final_convergence_Lipschitz}}{Corollary~\ref{main:final_convergence_Lipschitz}}}

Suppose $g$ is $L_g$-Lipschitz continuous. Then, from \eqref{eqn:smoothing-sandwich} we get
\begin{align}
    \Expect F(x_{k+1}) - \optimal F 
    & = \Expect [f(Xw_{k+1}) + g(Aw_{k+1})] - \optimal F \\
    & \leq \Expect [f(Xw_{k+1}) + g_{\beta_{k}}(Aw_{k+1})] - \optimal F + \frac{\beta_{k}L_g^2}{2} \\
    & = S_{\beta_{k}}(w_{k+1}) + \frac{\beta_0 L_g^2}{2\sqrt{k+1}}.
\end{align}

\subsection{Proof of \texorpdfstring{\Cref{main:final_convergence}}{Corollary~\ref{main:final_convergence}}}

Suppose $g(z) = \delta_{\constset}(z)$, the indicator function of a closed and convex set. 
We can write the Lagrangian as
\begin{equation}
    \mathcal{L}(w, r, y) \defeq f(Xw) + \langle Aw - r, y\rangle, \qquad w \in \domain, ~ r \in \constset.
\end{equation}
From the Lagrange saddle point theory,  we  have 
\begin{equation}
    f(X\optimal{w}) 
    \leq \mathcal{L}(w, r, \optimal{y}) \leq f(Xw) + \|Aw - r\| \|\optimal{y}\|, \qquad \forall w \in \domain ~\text{and}~ \forall r \in \constset.
\end{equation}
Letting $w = w_{k+1} \in \domain$ and $r = \proj_{\constset}(Aw_{k+1}) \in \constset$, taking expectation on both sides and rearranging, we get
\begin{equation}\label{eqn:supp-corr41-obj-lower-bound}
    \Expect\left[ f(Xw_{k+1}) - f(Xw^*)\right] \geq -\|\optimal{y}\| \, \Expect\left[\dist(Aw_{k+1}, \constset)\right]
\end{equation}
This is the desired lower-bound on objective residual.

Next, we derive an upper bound on objective residual. 
By definition of $g_\beta$ (see \eqref{eqn:supp-smoothing}) for $\delta_\constset$, 
\begin{equation}
g_{\beta} (Aw) 
= \frac{1}{2\beta} \dist(Aw, \mathcal{K})^2. 
\end{equation}
Note that $f(X\optimal{w}) = F(\optimal{w})$ since $g(A\optimal{w}) = 0$. Then, 
\begin{align}
\Expect[f(Xw_{k+1}) - f(X\optimal{w})] 
& = \Expect[F_\betak(w_{k+1}) - \optimal{F} - g_\betak(Aw_{k+1})] \\
& \leq S_\betak(w_{k+1}) - \frac{1}{2\betak}\Expect[\dist(Aw_{k+1}, \mathcal{K})^2] \label{eqn:supp-corr41-obj-upper-bound} \\
& \leq S_\betak(w_{k+1}).
\end{align}

Finally, we derive convergence rate of the infeasibility error. 
To this end, we combine \eqref{eqn:supp-corr41-obj-lower-bound} and \eqref{eqn:supp-corr41-obj-upper-bound}:
\begin{align}
-\|\optimal{y}\| \, \Expect\left[\dist(Aw_{k+1}, \constset)\right] 
\leq S_\betak(w_{k+1}) - \frac{1}{2\betak}\Expect[\dist(Aw_{k+1}, \mathcal{K})^2]
\end{align}
We rearrange and apply Jensen's inequality to $\Expect[\dist(Aw_{k+1},\constset)^2]$, and we get a second order inequality with respect to $\Expect[\dist(Aw_{k+1}, \mathcal{K})]$:
\begin{align}
\frac{1}{2\betak}\underbrace{\Expect[\dist(Aw_{k+1}, \mathcal{K})]^2}_{t^2} -\|\optimal{y}\| \, \underbrace{\Expect[\dist(Aw_{k+1}, \constset)]}_{t} - S_\betak(w_{k+1})
\leq 0.
\end{align}
By solving this inequality for $t$, we achieve the desired bound:
\begin{align}
\Expect[\dist(Aw_{k+1}, \constset)] 
\leq \betak\left( \|\optimal{y}\| + \sqrt{\|\optimal{y}\|^2 + \frac{2 S_\betak(w_{k+1})}{\betak}} \right) 
\leq 2\betak\|\optimal{y}\| + \sqrt{2 \betak  S_\betak(w_{k+1})} ,
\end{align}
where we used $\sqrt{a^2 + b^2} \leq a + b $ for $a, b \geq 0$ in the last inequality to simplify the terms.

\clearpage

\section{Proof of \texorpdfstring{\Cref{lem-gamma-error}}{Lemma~\ref{lem-gamma-error}}}\label{app:lem-gamma-error-proof}

The following Lemma will be needed in the subsequent characterization of the estimator variance.

\begin{lemma}
\label{lem:negiar--adapted}
Let $\rho\in (0,1)$, $C\in\mathbb R$ and $\{u_k\}_{k \in \N}$ be a sequence such that
\begin{equation}\label{eqn:supp-lemmaE1-cond}
    u_k \leq \rho(u_{k-1} + \frac{1}{\sqrt{k}} C).
\end{equation}
Then, it holds that
\begin{equation}\label{eqn:supp-lemmaE1}
u_k \leq \rho^{k} u_1 + \frac{2C\rho}{\sqrt{k}(1-\rho)} .    
\end{equation}
\end{lemma}

\begin{proof}
Unrolling the recurrence yields
\begin{align}\label{eqn:supp-lemmaE1-proof1}
   u_k 
   \leq \rho^{k-1} u_1 + C \sum_{i=2}^k \frac{\rho^{k-i+1}}{\sqrt{i}} 
\end{align}

Observe that $\rho^{k+1-i}$ is a monotonically increasing with $i$ because $\rho \in (0,1)$. Therefore, 
\begin{align} 
\frac{1}{\sum_{i=1}^k \frac{1}{\sqrt{i}}} \sum_{i=1}^k \frac{\rho^{k-i+1}}{\sqrt{i}} 
\leq \frac{1}{k} \sum_{i=1}^k \rho^{k-i+1} 
= \frac{1}{k} \sum_{i=1}^k \rho^{i} 
\end{align}
since the left side of the inequality is a weighted average of $\rho^{k-i+1}$ with decreasing weights and the right side is the simple average with uniform weights. The equality holds simply by change of indices. 
Now, we rearrange as
\begin{align} \label{eqn:supp-lemmaE1-proof2}
\sum_{i=1}^k \frac{\rho^{k-i+1}}{\sqrt{i}} 
\leq \frac{1}{k} \left(\sum_{i=1}^k \frac{1}{\sqrt{i}} \right) \left( \sum_{i=1}^k \rho^{i} \right)
\leq \frac{2\rho}{\sqrt{k}(1-\rho)}
\end{align}
We complete the proof by combining \eqref{eqn:supp-lemmaE1-proof1} and \eqref{eqn:supp-lemmaE1-proof2}.
\end{proof}

\subsection{Proof of \texorpdfstring{\Cref{lem-gamma-error}}{Lemma~\ref{lem-gamma-error}} for indicator functions}

First, we prove \Cref{lem-gamma-error} for the case in which $g$ is an indicator function. Observe that
\begin{align}
    \Expect_k[|\nabla g_{\beta_{k}}(Aw_k)_j - \gamma_{k, j}|] &= \frac{1}{m}0 + \frac{m-1}{m} |\nabla g_{\beta_{k}}(Aw_k)_j - \gamma_{k-1, j}|.
\end{align}
Summing over all coordinates gives
\begin{align} 
\Expect[\|\nabla g_{\beta_{k}}(Aw_k) - \gamma_k\|_1] 
&= \frac{m-1}{m} \Expect[\|\nabla g_{\beta_{k}}(Aw_k) - \gamma_{k-1}\|_1] \\
&=\frac{m-1}{m} \Expect[\|\nabla g_{\beta_{k}}(Aw_k) - \nabla g_{\beta_{k-1}}(Aw_{k-1}) + \nabla g_{\beta_{k-1}}(Aw_{k-1}) - \gamma_{k-1}\|_1]  \\
&\leq \frac{m-1}{m} \Big ( \Expect[\|\nabla g_{\beta_{k-1}}(Aw_{k-1}) - \gamma_{k-1}\|_1] + \Expect[\|\nabla g_{\beta_{k}}(Aw_k) - \nabla g_{\beta_{k-1}}(Aw_{k-1})\|_1] \Big).
\label{eqn:supp-lemma43-proof-1}
\end{align}
Now, we focus on the last term and bound it as follows:
\begin{align}
    \ellone{\nabla g_{\beta_{k}}(Aw_k) - \nabla g_{\beta_{k-1}}&(Aw_{k-1})}
    =
    \ellone{\nabla g_{\beta_{k}}(Aw_k) \pm \nabla g_\betak(Aw_{k-1}) - \nabla g_{\beta_{k-1}}(Aw_{k-1})}\\
    &\leq
    \ellone{\nabla g_{\beta_{k}}(Aw_k) - \nabla g_\betak(Aw_{k-1})} + \ellone{\nabla g_\betak(Aw_{k-1}) - \nabla g_{\beta_{k-1}}(Aw_{k-1})}\label{eq:expand_gradient_difference} \\
    &\leq
    \frac{1}{m\betak}\ellone{A(w_{k-1} - w_k)} + \frac{1}{m}\Big(\frac{1}{\betak}-\frac{1}{\betakminusone}\Big)\ellone{Aw_{k-1} - \proj_K(Aw_{k-1})}\\
    &\leq
    \frac{\eta_{k-1}}{m\betak}D_1(A) + \frac{1}{m}\Big(\frac{1}{\betak}-\frac{1}{\betakminusone}\Big)\ellone{Aw_{k-1} - A\optimal{w}}\\
    &\leq
    \frac{D_1(A)}{m}\Big(\frac{\eta_{k-1}}{\betak} + \frac{1}{\betak} - \frac{1}{\betakminusone}\Big)\label{eq:etak_betak_bound}
\end{align}
where the third inequality is due to the fact that $\constset = \constset_1\times \constset_2\times\cdots\times \constset_m$. Simplifying further:
$\frac{\eta_{k-1}}{\betak} + \frac{1}{\betak} - \frac{1}{\betakminusone}
=
\frac{2}{k}\frac{\sqrt{k+1}}{\beta_0} + \frac{\sqrt{k+1}}{\beta_0} - \frac{\sqrt{k}}{\beta_0}
<
\frac{2}{k}\frac{\sqrt{k}+1}{\beta_0} + \frac{\sqrt{k}\sqrt{k+1}}{\beta_0\sqrt{k}} - \frac{k}{\beta_0\sqrt{k}}
< 
\frac{2}{\beta_0\sqrt{k}} + \frac{2}{\beta_0k} + \frac{k+1}{\beta_0\sqrt{k}} - \frac{k}{\beta_0\sqrt{k}}
<
\frac{5}{\beta_0\sqrt{k}}$, gives
\begin{align}
\|\nabla g_{\beta_{k}}(Aw_k) - \nabla g_{\beta_{k-1}}(Aw_{k-1})\|_1
\leq
\frac{5D_1(A)}{m\beta_0\sqrt{k}}.
\end{align}

Substituting this back into \eqref{eqn:supp-lemma43-proof-1}, we get
\begin{align} 
\Expect[\|\nabla g_{\beta_{k}}(Aw_k) - \gamma_k\|_1] 
&\leq \frac{m-1}{m} \Big ( \Expect[\|\nabla g_{\beta_{k-1}}(Aw_{k-1}) - \gamma_{k-1}\|_1] + \frac{5D_2(A) \sqrt{m}}{\beta_0 \sqrt{k}}  \Big).
\end{align}
This is in the form of \eqref{eqn:supp-lemmaE1-cond}. We conclude the proof by applying \Cref{lem:negiar--adapted}:
\begin{equation}
\Expect[\|\nabla g_{\beta_{k}}(Aw_k) - \gamma_k\|_1] \leq \left( \frac{m-1}{m} \right)^{k} \Expect[\|\nabla g_{\beta_{0}}(Aw_0) - \gamma_0\|_1] + \frac{10 D_2(A) \sqrt{m} (m-1) }{\beta_0\sqrt{k}} .    
\end{equation}

\subsection{Proof of \texorpdfstring{\Cref{lem-gamma-error}}{Lemma~\ref{lem-gamma-error}} for Lipschitz continuous functions}

Suppose $g$ is Lipschitz continuous with parameter $L_g$. Then, from \eqref{eqn:smoothing-sandwich}, we get
\begin{equation}
\underbrace{f(Xw_{k+1}) + g (Aw_{k+1})}_{F(w_{k+1})}
\leq \underbrace{f(Xw_{k+1}) + g_{\betak} (Aw_{k+1})}_{F_\betak(w_{k+1})} + \frac{\betak}{2}  L_g^2
= F_\betak(w_{k+1}) + \frac{\beta_0 L_g^2}{2\sqrt{k+1}}.
\end{equation}
We achieve the desired bound by subtracting $\optimal{F}$ and taking expectation on both sides:
\begin{equation}
\Expect[F(w_{k+1}) - \optimal{F}]
\leq S_\betak(w_{k+1}) + \frac{\beta_0 L_g^2}{2\sqrt{k+1}}.
\end{equation}

To bound $S_\betak$, we can follow the proof of \Cref{lem-gamma-error} up to \eqref{eq:expand_gradient_difference}, which we repeat here for convenience:
\begin{equation*}
\ellone{\nabla g_{\beta_{k}}(Aw_k) - \nabla g_\betak(Aw_{k-1})} + \ellone{\nabla g_\betak(Aw_{k-1}) - \nabla g_{\beta_{k-1}}(Aw_{k-1})}
\end{equation*}
Recall that $\grad g_\beta(z) = \beta^{-1}(z - \prox_{\beta g}(z))$. The first term can be bounded using the $1/\beta$-smoothness of $g_\beta$. For the second term, recall the well-established fact that $\prox_g(z) = \lambda \prox_{g/\lambda}(x/\lambda)$ for any $\lambda > 0$. Thus, 
\begin{align}
    \grad g_\betak(Aw_{k-1}) &= \betak^{-1} (Aw_{k-1} - \prox_{\betak g}(Aw_{k-1}))\\
    &= \betak^{-1}(Aw_{k-1} - \frac{\betak}{\betakminusone}\prox_{\betakminusone g}(\frac{\betakminusone}{\betak}Aw_{k-1})) \\
    &= \grad g_\betakminusone(\frac{\betakminusone}{\betak}Aw_{k-1})
\end{align}
Thus,
\begin{align}
&\ellone{\nabla g_{\beta_{k}}(Aw_k) - \nabla g_\betak(Aw_{k-1})} + \ellone{\nabla g_\betak(Aw_{k-1}) - \nabla g_{\beta_{k-1}}(Aw_{k-1})}\\
&\leq
\frac{1}{m\betak}\ellone{A(w_k - w_{k-1})}
+ \frac{1}{m\betakminusone}(\frac{\betakminusone}{\betak}-1)\ellone{Aw_{k-1}}\\
&\leq
\frac{\eta_{k-1}}{m\betak}D_1(A)
+ \frac{1}{m}(\frac{1}{\betak} - \frac{1}{\betakminusone})\ellone{Aw_{k-1}}\\
&\leq
\frac{D_1(A)}{m}\left(\frac{\eta_{k-1}}{\beta_k} + \frac{1}{\betak} - \frac{1}{\betakminusone}\right)
\end{align}
Note that this is identical to 
\eqref{eq:etak_betak_bound} in 
\Cref{lem:negiar--adapted}. Thus, the rest of \Cref{lem:negiar--adapted} can be applied to arrive at the same bound.

\clearpage

\section{Proof of \texorpdfstring{\Cref{main:smoothed_gap_convergence}}{Theorem~\ref{main:smoothed_gap_convergence}} for H-SAG-CGM/v2}

The proof is same until \eqref{eqn:supp-thm41-ABbound}. Then, get an upper-bound on the variance term {\smash{\circled{A}}} as follows: 
\begin{align}
\Expect[\innerp{\grad F_\betak(w_k) - v_k, & s_k - \optimal{w}}] 
= \Expect[\innerp{X^T(\grad f(Xw_k) - \alpha_k) + A^T(\grad g_\betak(Aw_k) - \gamma_k), s_k - \optimal{w}}] \\
& = \Expect[\innerp{\grad f(Xw_k) - \alpha_k, X(s_k - \optimal{w})} +  \partB{\innerp{\grad g_\betak(Aw_k) - \gamma_k, A(s_k - \optimal{w})}}] \\
& \leq \Expect[\norm{\grad f(Xw_k) - \alpha_k}_1 \,\norm{X(s_k - \optimal{w})}_\infty + \partB{\norm{\grad g_\betak(Aw_k) - \gamma_k}_1 \,\norm{A(s_k - \optimal{w})}_\infty}] \\
& \leq \Expect[\norm{\grad f(Xw_k) - \alpha_k}_1] \, D_{\infty}(X) + \partB{\Expect[\norm{\grad g_\betak(Aw_k) - \gamma_k}_1] \, D_{\infty}(A)}
\label{eqn:supp-thm42-variance-1}
\end{align}
where, the first inequality is the H\"older's inequality, and the second one is based on the boundedness of $\domain$.

Then, by \Cref{main:alpha_err}, we have  
\begin{align}
\Expect[\|\grad f(Xw_k) - \alpha_k\|_1] 
\leq \left(1 - \tfrac{1}{n}\right)^{k} \|\grad f(Xw_1) - \alpha_0\|_1 + \frac{2 L_f D_1(X)}{n} \left(  \left(1 - \tfrac{1}{n}\right)^{k/2}  \log k + \frac{2(n-1)}{k} \right)
\label{eqn:supp-thm42-variance-f}
\end{align}
And by \Cref{lem-gamma-error}, we have
\begin{align}
\partB{\Expect[\|\nabla g_{\beta_{k}}(Aw_k) - \gamma_k\|_1] \leq \left( 1-\tfrac{1}{m} \right)^{k} \Expect[\|\nabla g_{\beta_{0}}(Aw_1) - \gamma_0\|_1] + \frac{10 D_2(A) \sqrt{m} (m-1) }{\beta_0\sqrt{k}}.}
\label{eqn:supp-thm42-variance-g}
\end{align}

Finally, we substitute \eqref{eqn:supp-thm42-variance-f} and \eqref{eqn:supp-thm42-variance-g} back into \eqref{eqn:supp-thm42-variance-1} to get
\begin{align}
\circled{A} 
& \leq 2 D_{\infty}(X) \left[ \|\grad f(Xw_1) - \alpha_0\|_1 \sum_{i=1}^k i \left(1 - \tfrac{1}{n}\right)^{i}  + \frac{2 L_f D_1(X)}{n} \sum_{i=1}^k\left( i \left(1 - \tfrac{1}{n}\right)^{i/2}  \log i + 2(n-1) \right) \right] \notag \\
& \qquad + \partB{2 D_\infty(A) \left[ \norm{\grad g_{\beta_0}(Aw_1) - \gamma_0}_1 \sum_{i=1}^k i \left(1 - \tfrac{1}{m}\right)^{i} + \frac{10 D_2(A) \sqrt{m} (m-1) }{\beta_0} \sum_{i=1}^k \sqrt{i} \right]} \\[0.75em]
& \leq 2 D_{\infty}(X) \left[ \|\grad f(Xw_1) - \alpha_0\|_1 \, n^2  + \frac{2 L_f D_1(X)}{n} \left( 16 n^3 + 2(n-1)k \right) \right] \notag \\
& \qquad + \partB{2 D_\infty(A) \left[ \norm{\grad g_{\beta_0}(Aw_1) - \gamma_0}_1 \, m^2 + \frac{10 D_2(A) \sqrt{m} (m-1) }{\beta_0} k^{3/2} \right]} \\[1em]
& \leq 2 D_{\infty}(X) \left[ \|\grad f(Xw_1) - \alpha_0\|_1 \, n^2  + 4 L_f D_1(X) \left( 8 n^2 + k \right) \right] \notag \\[0.5em]
& \qquad + \partB{2 D_\infty(A) \left[ \norm{\grad g_{\beta_0}(Aw_1) - \gamma_0}_1 \, m^2 + \frac{10 D_2(A) \, m^{3/2}}{\beta_0} k^{3/2} \right]} 
\end{align}
where we use \Cref{lem:bounds-for-sums} for the second inequality. 

Combining this with the bound on the smoothness term \smash{\circled{B}} from \eqref{eqn:supp-thm41-ABbound} gives the desired result:
\begin{align*}
    S_\betak(w_{k+1})
    & \leq \frac{2 D_{\infty}(X)}{k(k+1)}  \Bigg\{ \|\grad f(Xw_1) - \alpha_0\|_1 \, n^2  + 4 L_f D_1(X) \left( 8 n^2 + k \right) \\
    &\quad + \partB{2 D_\infty(A) \left[ \norm{\grad g_{\beta_0}(Aw_1) - \gamma_0}_1 \, m^2 + \frac{10 D_2(A) \, m^{3/2}}{\beta_0} k^{3/2} \right]}
    \Bigg\} \\
    &+ \frac{2 D_\domain^2}{k(k+1)}  \left( \frac{\|X\|L_f}{n} k +  \frac{\|A\|}{\beta_0} k \sqrt{k+1} \right)  \\
    & \leq \frac{C_3}{k(k+1)} + \frac{C_2}{k+1} + \frac{C_1}{\sqrt{k+1}}, \quad \text{where} \quad 
    \begin{aligned}[t]
    C_3 & = 2 n^2 D_{\infty}(X) \big( \|\grad f(Xw_1) - \alpha_0\|_1  + 32 L_f D_1(X) \big) \\
    &\qquad+ 2 m^2 D_\infty(A)\|\nabla g_{\beta_0}(Aw_1) - \gamma_0\|_1 \\
    C_2 & = 8 L_f D_1(X) D_{\infty}(X) + 2 n^{-1} L_f \|X\| D_\domain^2  \\
    C_1 & = \beta_0^{-1}(2 D_\domain^2 \|A\| + 10 D_1(A)).
    \end{aligned}
\end{align*}

\clearpage

\section{Supporting Lemmas}

\begin{lemma}
\label{lem:bounds-for-sums}
Let $\rho_n = 1 - \frac{1}{n}$ and $\rho_m = 1 - \frac{1}{m}$, $m, n \geq 1$. We present the following bounds:

\begin{enumerate}[label=\alph*)]
    \item $\sum\limits_{i=1}^k i\rho_n^{i} < n^2 \quad$ and $\quad \sum\limits_{i=1}^k i\rho_m^{i} < m^2$
    \item $\sum\limits_{i=1}^k i\rho_n^{i/2}\log i < 16n^3$
\end{enumerate}

\end{lemma}\label{app:bounded_seqs}

\begin{proof}

\textbf{a)} Note that since $\rho_n \in [0,1)$, $\sum_{i=1}^k i\rho_n^{i} \leq \sum_{i=1}^k i\rho_n^{i-1}$. Furthermore, 
\begin{align}
    \sum_{i=1}^k i\rho_n^{i-1} &\leq \sum_{i=1}^{\infty} i\rho_n^{i-1} = \sum_{i=1}^\infty \frac{\partial \rho_n^{i}}{\partial \rho_n} = \frac{\partial \sum\limits_{i=1}^\infty  \rho_n^{i}}{\partial \rho_n} = \frac{\partial \left[\frac{1}{1-\rho_n} - 1\right]}{\partial \rho_n} = \frac{1}{(1-\rho_n)^2} = n^2,
\end{align}
where the inequality comes from all terms being non-negative, and the second equality comes from the fact that the infinite sum exists for any $\rho_n \in (-1, 1)$ and is the Taylor series expansion of $\frac{1}{1-\rho_n}$. 

\medskip

\textbf{b)} Use the loose bound $\log i < i+1$ and the fact that $\sqrt{\rho_n} \in [0,1)$:
\begin{align}
    \sum_{i=1}^k i\rho_n^{i/2}\log i
    \leq \sum_{i=1}^\infty i\rho_n^{i/2}\log i &\leq \sum_{i=1}^\infty i (i+1) \sqrt{\rho_n}^{i-1} \\
    &= \frac{\partial^2\sum_{i=2}^\infty  \sqrt{\rho_n}^{i}}{\partial (\sqrt{p_n})^2} = \frac{\partial^2\frac{1}{1-\sqrt{\rho_n}} - \sqrt{\rho_n} - 1}{\partial (\sqrt{p_n})^2} = \frac{2}{(1 - \sqrt{\rho_n})^3}
\end{align}
where the inequalities and equalities follow the same reasoning as in point a). Further noting that 
\begin{align}
    \frac{2}{(1 - \sqrt{\rho_n})^3} =  \frac{2(1 +\sqrt{\rho_n} )^3}{(1 - \rho_n)^3} = 2n^3(\underbrace{1 +\sqrt{\rho_n}}_{\leq 2} )^3 \leq 16 n^3
\end{align}

\end{proof}

\section{Uniform Sparsest Cut Datasets}

\begin{table}[H]
    \centering
    \caption{Datasets used for Uniform Sparsest Cut experiments.
    ``Deg.'' stands for ``Degree,'' ``\# Constraints'' refers to the number of constraints in the SDP relaxation, and the dimension of the decision variable $w$ is $n\times n$.}
    \vspace{1.0em}  %
    \begin{tabular}{r|cccc|c}
        & $|V|=n$ & $|E|$ & Avg.\,Node\,Deg. & Max.\,Node\,Deg. & \# Constraints \\\midrule
        \texttt{mammalia-primate-association-13}     & $25$ & $181$ & $14$ & $19$ & $\approx 6.90\mathrm{e}{3}$ \\
        \texttt{insecta-ant-colony1-day37} & $55$ & $1\mathrm{e}{3}$ & $42$ & $53$ & $\approx 7.87\mathrm{e}{4}$ \\
        \texttt{insecta-ant-colony4-day10} & $102$ & $4\mathrm{e}{3}$ & $79$ & $99$ & $\approx 5.15\mathrm{e}{5}$ \\
    \end{tabular}
    \label{tab:usc_graphs}
\end{table} 
\end{document}